\documentclass[sigconf]{acmart}

\AtBeginDocument{%
  \providecommand\BibTeX{{%
    \normalfont B\kern-0.5em{\scshape i\kern-0.25em b}\kern-0.8em\TeX}}}

\newtheorem{theorem}{Theorem}
\usepackage{multirow}
\usepackage{subfigure}
\usepackage{textcomp}
\usepackage{graphicx}
\usepackage{amsmath}
\usepackage{amsthm}
\usepackage{booktabs}
\usepackage{algorithm}
\usepackage{algorithmic}
\usepackage{textcomp}

\setcopyright{acmcopyright}

\begin{document}

\title{Adaptive-Step Graph Meta-Learner for Few-Shot Graph Classification}

\author{Ning Ma}
\affiliation{\institution{College of Computer Science, Zhejiang University}}
\email{ma ning@zju.edu.cn}

\author{Jiajun Bu}
\affiliation{\institution{College of Computer Science, Zhejiang University}}
\email{bjj@zju.edu.cn}

\author{Jieyu Yang}
\affiliation{\institution{College of Computer Science, Zhejiang University}}
\email{	yangjieyu@zju.edu.cn}

\author{Zhen Zhang}
\affiliation{\institution{College of Computer Science, Zhejiang University}}
\email{zhen_zhang@zju.edu.cn}

\author{Chengwei Yao}
\affiliation{\institution{College of Computer Science, Zhejiang University}}
\email{yaochw@zju.edu.cn}

\author{Zhi Yu}
\affiliation{\institution{College of Computer Science, Zhejiang University}}
\email{yuzhirenzhe@zju.edu.cn}

\author{Sheng Zhou}
\affiliation{\institution{College of Computer Science, Zhejiang University}}
\email{zhousheng_zju@zju.edu.cn}

\author{Xifeng Yan}
\affiliation{\institution{University of California, Santa Barbara}}
\email{xyan@cs.ucsb.edu}



\begin{abstract}
  Graph classification aims to extract accurate information from graph-structured data for classification and is becoming more and more important in graph learning community. 
  Although Graph Neural Networks (GNNs) have been successfully applied to graph classification tasks, most of them overlook the scarcity of labeled graph data in many applications. 
  For example, in bioinformatics, obtaining protein graph labels usually needs laborious experiments. 
  Recently, few-shot learning has been explored to alleviate this problem with only a few labeled graph samples of test classes.
  The shared sub-structures between training classes and test classes are essential in few-shot graph classification. Existing methods assume that the test classes belong to the same set of super-classes clustered from training classes.
  However, according to our observations, the label spaces of training classes and test classes usually do not overlap in real-world scenario.
  As a result, the existing methods don't well capture the local structures of unseen test classes.
  To overcome the limitation, in this paper, we propose a direct method to capture the sub-structures with well initialized meta-learner within a few adaptation steps.
 More specifically, (1) we propose a novel framework consisting of a graph meta-learner, which uses GNNs based modules for fast adaptation on graph data, and a step controller for the robustness and generalization of meta-learner; (2) we provide quantitative analysis for the framework and give a graph-dependent upper bound of the generalization error based on our framework; (3) the extensive experiments on real-world datasets demonstrate that our framework gets state-of-the-art results on several few-shot graph classification tasks compared to baselines.
\end{abstract}

\ccsdesc[500]{Computing methodologies}
\ccsdesc[500]{Computing methodologies~Neural networks}

\keywords{graph data mining, few-shot classification, meta-learning, graph neural networks}


\maketitle
\section{Introduction}
Many real-world networks can be formulated as graphs for modeling different relationships among nodes, such as social networks, chemical molecule structures and citation networks. Recently, there have been various attempts to extend Convolutional Neural Network (CNNs) and pooling methods to graph-structured data. These methods are named as Graph Neural Networks (GNNs) and have been successfully applied to different graph related tasks containing graph classification, node classification and link prediction \cite{zhang2018deep,zhou2018graph}.
Graph classification aims to extract accurate information from graph-structured data for classification. However, most existing GNNs based graph classification methods overlook that it's complicated and time consuming to collect or label the graph data. Learning with few labeled graph data is still a challenge for the practical applications of graph classification. 
\begin{figure*}[!t]\small

\centering
\subfigure[\scriptsize{Graph spectral measures method \cite{Chauhan2020FEW-SHOT}.}]{
\label{fig:gmsdemo}
\centering
\includegraphics[width=0.45\textwidth]{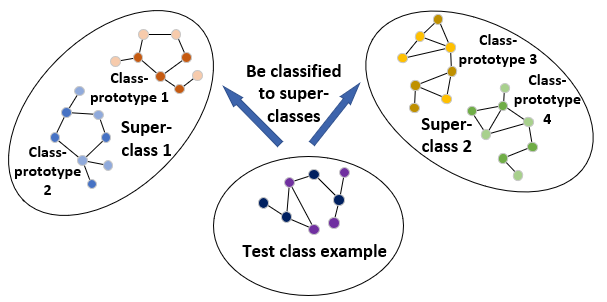}
}
\subfigure[\scriptsize{Our method.}]{
\label{fig:ourdemo}
\centering
\includegraphics[width=0.45\textwidth]{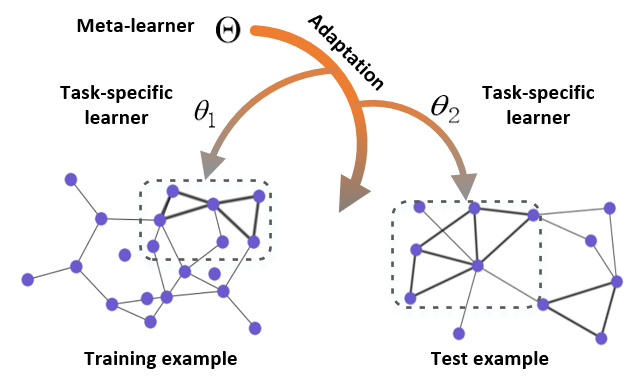}
}
\caption{Comparison of different methods. (a) The method from the global structure of dataset. The most representative graph of each class is viewed as class-prototype graph. The super-classes are clustered from training classes. (b) Our method from the view of local structure. $\Theta$ is meta-learner's parameters and $\theta_{1},\theta_{2}$ are task specific parameters derived from meta-learner's fast adaptation. The components in dotted boxes have similar triangle structure. We assume the similarity can be discovered by a well initialized meta-learner within a few adaptation steps.}
\label{fig:first_image}
\end{figure*} 

Few-shot learning, aiming to label the query data when given only a few labeled support data, is a natural way to alleviate the problem. There are many papers discussing few-shot learning with meta-learning \cite{FinnAL17},  data augmentation \cite{NIPS2018_7504} or regularization \cite{Yin2020Meta-Learning}, but most of them don't consider the graph data. 
Furthermore, there have been several methods for few-shot node classification \cite{SatorrasICLR2018graph,Kim_2019_CVPR,meta_gnn} and few-shot link prediction \cite{NIPS2017_7266,recommendation,MeLU,adCold-Start}, but they only focus on node-level embedding. Recently, Chauhan et al. \cite{Chauhan2020FEW-SHOT} proposed few-shot graph classification based on graph spectral measures and got satisfactory performance. 
From the global structure of dataset, they \textit{bridge} the test classes and training classes by assuming that the test classes belong to the same set of super-classes clustered from training classes. However, the above methods based on graph spectral measures might have some limitations for the following reasons: (1) the label spaces of training classes and test classes usually do not overlap in few-shot settings; (2) the bridging methods above may diminish the model to capture the local structure of test data.

From the perspective of the graph's local structure, we observe that the graphs of training classes and test classes have similar sub-structures. For example, different social networks usually have similar groups; different protein molecules usually have similar spike proteins. We assume these similarities can be discovered by a well initialized meta-learner within a few adaptation steps. 
Therefore, we consider fast adaptation by a meta-learner from learned graph classification tasks to new tasks. Figure \ref{fig:first_image} illustrates the existing method and our assumption. 

Currently, GNNs have reliable ability to capture local structures over graphs by convolutional operations and pooling operations, but lack of fast adaptation mechanism when dealing with never seen graph classes. Inspired by Model Agnostic Meta-Learning (MAML, \cite{FinnAL17}), which has attracted great attention because of its fast adaptation mechanism, we leverage GNNs as graph embedding backbone and meta-learning as a training paradigm to rapidly capture task-specific knowledge in graph classification tasks and transfer them to new tasks.

However, directly applying MAML for fast adaptation is suboptimal due to the following reasons: (1) MAML requires painstaking hyperparameter searches to stabilize training and achieve high generalization \cite{antoniou2018how}; (2) unlike images, graphs have arbitrary node size and sub-structure, which brings uncertainty for adaptation. There have some variants of MAML trying to overcome these problems by incorporating an online hyperparameter adaptation \cite{behl2019alpha}, reducing optimization difficulty \cite{nichol2018firstorder} or increasing context parameters for adaptation \cite{zintgraf2019fast}, but they don't consider the structure of graph data. In this paper, we design a novel component named as adaptive step controller to learn optimal adaptation step for meta-learner to improve its learning robustness and generalization. The controller evaluates the meta-learner and decides when to stop adaptation by two kinds of inputs: (1) graphs' embedding quality, which is viewed as a meta-feature and indicated with Average Node Information (ANI, the average amount of node information in a batch of graphs); (2) meta-learner's training state, which is indicated with training loss of classification. 

We formulate our framework as \textbf{A}daptive \textbf{S}tep MAML (AS-MAML). To the best of our knowledge, we are the first to consider the few-shot graph classification problem from the view of graph's local structure and propose a fast adaptation mechanism on graphs via meta-learning. Our contributions are summarized as follows:
\begin{itemize}
\item We propose a novel GNNs based graph meta-learner, which captures the features efficiently of sub-structures on unseen graphs by fast adaptation mechanism.
\item We design a novel controller for meta-learner. Driven by Reinforcement Learning (RL, \cite{rl}), the controller provide optimal adaptation step for the meta-learner via graph's embedding quality and training loss. Our ablation experiments show its effectiveness to improve learning robustness and generalization.
\item We perform quantitative analysis and provide a generalization guarantee of key algorithms via a graph-dependent upper bound.  

\item We evaluate our framework's performance against different baselines on four graph datasets and achieve state-of-the-art performance in almost all the tasks. We also evaluate the transferability of popular graph embedding modules on our few-shot graph classification tasks.
\end{itemize}
\section{Related Works}
\subsection{Graph Classification}

In graph classification tasks, each full graph is assigned a class label. There exist several branches for graph classification. The first is graph kernel methods which design kernels for the sub-structures exploration and exploitation of graph data. The typical kernels include Shortest-path Kernel \cite{borgwardt2005shortest}, Graphlet Kernel \cite{shervashidze2009efficient} and Weisfeiler-Lehman Kernel \cite{shervashidze2011weisfeiler}.

As the main branch in recent years, GNNs have been successfully applied to graph classification. GNNs focus on node representations, which are iteratively computed by message passing from the features of their neighbor nodes using a differentiable aggregation operation.  GCN \cite{kipf2016semi} proposed Graph Convolutional Neural Network (termed as  GCN) and got satisfying results based on directly feature aggregation from neighborhood nodes. GAT \cite{GAT} imported attention mechanism for graph convolutional operations. GraphSAGE \cite{SAGE} proposed an inductive framework which leverages node features to generate node embeddings efficiently for unseen nodes. In our framework, we use these classical methods to update nodes of graphs, while other methods like Graph Isomorphism Network (GIN)   \cite{xu2018how} are also applicable. 

In the meantime, inspired by pooling in CNNs, a bunch of researchers concentrates on efficient pooling methods for accurate graph summary and computation efficiency.  Beyond pooling layers in CNNs, graph pooling layers can enable GNNs to reason and get global representation from adjacent nodes. More and more evidence shows that graph pooling promotes the graph classification performance  \cite{lee2019self,diehl2019edge,unet}. SAGPool \cite{lee2019self} implemented self-attention pooling on graphs considering both node features and graph topology. EdgePool \cite{diehl2019edge} implemented a localized and sparse pooling transform backed by the notion of edge contraction. Graph U-nets \cite{unet} implemented novel graph pooling and unpooling operations. Based on these operations, they developed a new model containing graph encoder and graph decoder and got satisfactory performance on graph classification tasks.

\subsection{Few-Shot Learning and Meta-Learning}
Few-shot classification aims to learn a model under the circumstances of low sample resources and is usually powered by meta-learning in recent years. Meta-learning was also known as learning to learn, with a meta-learner observing various task learning processes and summarizing meta-knowledge to accelerate the learning efficiency of new tasks.  Baxter et al.  \cite{BaxterBiasLearning} proposed a model to learn inductive bias from the perspective of bias learning, and they analytically showed that the number of examples required of each task decreases as the number of task rises.

Recent meta-learning related works can be classified into three categories: optimization (or gradients) based methods, metric learning based methods and memory based methods. Optimization based methods aim to train a model to learn optimization \cite{RaviICLR2017,LiICML2018}, learn a good initialization \cite{FinnAL17} for rapid adaptation, or train parameter generator for task-specific classifier \cite{RusuICLR2019}. Metric learning based methods aim to learn a feature space shared with new tasks \cite{VinyalsBLKW16,SnellSZ17}. Moreover, memory based methods learn new tasks by reminiscence mechanism in virtue of physical memory \cite{SantoroBBWL16}.

Furthermore, almost all the previous few-shot learning methods are devised for image data, where images are prone to be represented in Euclidean space. Because we all have the idea that CNNs based models can perform efficient transfer in Euclidean space by feature reuse \cite{Raghu2020Rapid}, in virtue of that different images usually share common edge features and corner features.  Graph data such as social networks, which are appropriate to be formed into non-Euclidean space instead of Euclidean space. Few-shot learning in non-Euclidean space is addressed in our work.
\subsection{GNNs and its Generalization on Graph data}
We have seen several works of few-shot node classification promoting performance via GNNs \cite{SatorrasICLR2018graph, Kim_2019_CVPR, liu2019GPN, yang2020dpgn,Yao2020Automated,NIPS2019_8389,yao2019graph,Gidaris_2019_CVPR,Liu2019PrototypePN}, but they just leverage the message passing mechanism of GNNs to enhance the performance on node classification, without involving the generalization of GNNs themselves and compatibility with graph classification task.  For graph classification, Knyazev et al.  \cite{NIPS2019_8673} focus on the ability of attention GNNs to generalize to larger, more complex or noisy graphs. Lee et al. \cite{LeeTransfer} imported domain transfer method by transferring the intrinsic geometric information learned in the source domain to the target. Hu et al. \cite{Hu*2020Strategies} systematically studied the effectiveness of pre-training strategies on multiple graph datasets. Based on graph spectral measures, Chauhan et al. \cite{Chauhan2020FEW-SHOT} proposed few-shot graph classification using the notion of super-graph by two steps: (1) they define the \(p\)-th Wasserstein distance to measure the spectral distance among graphs and select the most representative graph as prototype graph for each class; (2) by clustering the prototype graphs based on spectral distance, they clustered the prototype graph again into a super-graph consisting of super-classes. Therefore, they assume that the test classes belong to the same set of super-classes clustered from the training classes. We loosen the assumption and emphasize fast adaptation to boost few-shot graph classification. 

\section{Problem Setup}
\begin{figure*}[!ht]
  \centering
  \includegraphics[width=0.70\textwidth]{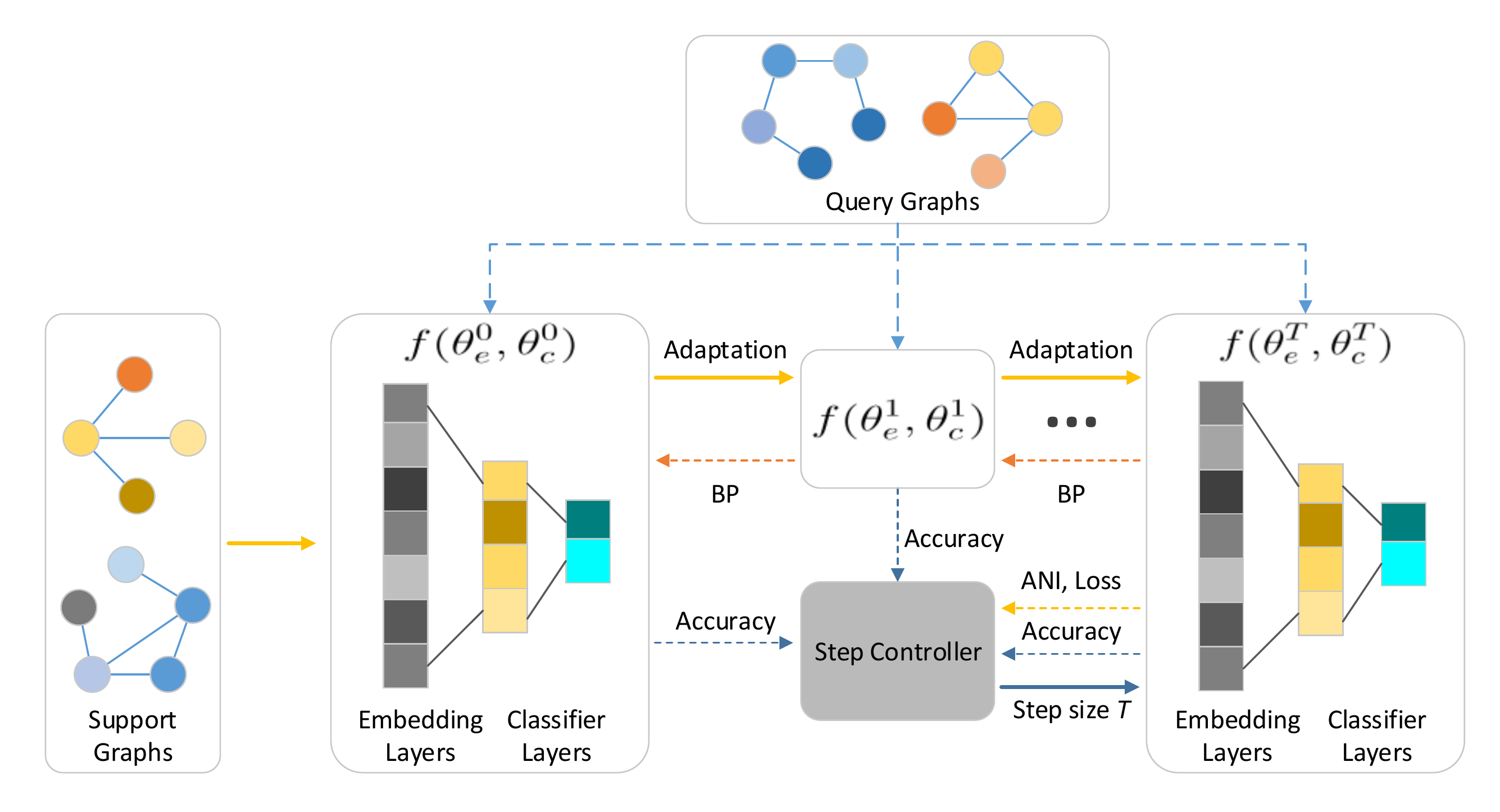}
  \caption{Diagram of the AS-MAML framework's learning process in a single episode on the 2-way-1-shot graph classification task. The yellow arrows show meta-learner's $T$ step adaptations on support graphs. The blue dash arrows show $T$ step evaluations (Accuracies) on the query graphs. The orange dash arrows show the backpropagation (BP) according to $T$-th loss on query graphs. The step controller receives ANIs and classification losses on support graphs of each step. After that, the controller outputs the adaptation step $T$. Finally, the controller receives accuracies on query graphs as rewards and updates its own parameters. }
  \label{fig:fm}
\end{figure*}

We form the few-shot problem as N-way-K-shot graph classification. Firstly, given graph data $\mathcal{G}=\left\{({G}_{1},\mathbf{y}_{1}), ({G}_{2},\mathbf{y}_{2}), \cdots, ({G}_{n},\mathbf{y}_{n})\right\}$, where ${G}_{i}= \left(\mathcal{V}_{i}, \mathcal{E}_{i}, \mathbf{X}_{i}\right)$. We use ${n}_{i}$ to denote the number of node set $ \mathcal{V}_{i}$. So each graph ${G}_{i}$ has an adjacent matrix  $\mathbf{A}_{i} \in \mathbb{R}^{n_{i} \times n_{i}}$ and a node attribute matrix $\mathbf{X}_{i} \in \mathbb{R}^{n_{i} \times d}$, where $d$ is the dimension of node attribute. Secondly, according to label $\mathbf{y}$, we split $\mathcal{G}$ into $\{(\mathcal{G}^{train},\mathbf{y}^{train})\}$ and $\{(\mathcal{G}^{test},\mathbf{y}^{test})\}$ as training set and test set respectively. Notice that $\mathbf{y}^{train}$ and $\mathbf{y}^{test}$ must have no common classes. We use episodic training method, which means at the training stage we sample a task $\mathcal{T}$ each time, and each task contains support data $D^{train}_{sup}={ \{ (G_{i}^{train},\mathbf{y}_{i}^{train}) \} }_{i=1}^{s}$ and query data $D^{train}_{que}={ \{ (G_{i}^{train},\mathbf{y}_{i}^{train}) \} }_{i=1}^{q}$,
where $s$ and $q$ are the number of support data and query data respectively. Given labeled support data, our goal is to predict the labels of query data. Please note that in a single task, support data and query data share the same class space. If $s={N \times K}$, which means that support data contain N classes and K labeled samples per class, we name the problem as N-way-K-shot graph classification. At test stage when performing classification tasks on unseen classes, we firstly fine tune the meta-learner on the support data of test classes $D^{test}_{sup}={ \{ (G_{i}^{test},\mathbf{y}_{i}^{test}) \} }_{i=1}^{s}$ , then we report classification performance on  $D^{test}_{que}={ \{ (G_{i}^{test},\mathbf{y}_{i}^{test}) \} }_{i=1}^{q}$.
\section{Proposed Framework}

Overall, our few-shot graph classification framework consists of GNNs based meta-learner and a step controller to decide the adaptation steps of meta-learner. We use MAML to implement a fast adaptation mechanism for meta-learner because of its model agnostic property. Du et al. \cite{recommendation} proposed an RL based step controller to guide meta-learner for link prediction. We argue that classification loss is suboptimal to be viewed as rewards for overcoming overfitting. Therefore, we adopt a novel step controller to accelerate training and overcome overfitting. Our step controller is also driven by RL but learns the optimal adaptation step by using ANIs and losses as inputs and classification accuracy as rewards. Figure \ref{fig:fm} demonstrates the training process of our framework.

\subsection{Graph Embedding Backbone }
We explain our proposed framework with typical graph convolutional modules and pooling modules as embedding backbone, due to that novel graph convolutional modules or pooling modules are out of concern for this paper. The first step to represent a graph is to embed the nodes it contains. We investigate several embedding methods such as GCN, GAT, GraphSAGE and GIN. Here we focus on GraphSAGE as following reasons: (1) GraphSAGE has more flexible aggregators in few-shot learning scenarios; (2) Errica et at. \cite{Errica2020A} set GraphSAGE as a strong baseline when compared to GIN for graph classification task. Hence we use mean aggregator of GraphSAGE as follows:   
\begin{equation}
\mathbf{h}_{v}^{l} = \sigma\left(\mathbf{W} \cdot \operatorname{mean}\left(\left\{\mathbf{h}_{v}^{l-1}\right\} \cup\left\{\mathbf{h}_{u}^{l-1}, \forall u \in \mathcal{N}(v)\right\}\right)\right.,
\end{equation}
where $\mathbf{h}_{v}^{l}$ is the $l$-th layer representation of node $v$, $\sigma$ is the sigmoid function, $\mathbf{W}$ is the parameters and $\mathcal{N}(v)$ contains the neighborhoods of $v$.  Please note that this mean aggregator just belongs to the group of typical aggregators we use in experiments. We will provide concrete analysis for other aggregators in Section \ref{section:bound}  and Section \ref{section:exp}.
%

After that, we discuss existing pooling operations. Under the circumstances of few-shot learning, the meta-learner urgently needs a flexible pooling strategy with learning capability to strengthen its generalization. Here, we focus on self-attention pooling (SAGPool) \cite{lee2019self} as our pooling layer thanks to its flexible attention parameters. The main step of SAGPool is to calculate the attention score matrix of graph ${G}_{i}$ as follows: 
\begin{equation}
\mathbf{S}_{i}=\sigma\left(\mathbf{\tilde{D}}_{i}^{-\frac{1}{2}} \mathbf{\tilde{A}}_{i} \mathbf{\tilde{D}}_{i}^{-\frac{1}{2}} \mathbf{X}_{i} \mathbf{\Theta}_{a t t}\right),
\end{equation}
where the $\mathbf{S}_{i} \in \mathbb{R}^{n_{i} \times 1}$  indicates the self-attention score, $n_{i}$ is node number of the graph. $\sigma$ is the activation function (e.g., tanh), $\mathbf{\tilde{A}}_{i} \in \mathbb{R}^{n_{i} \times n_{i}}$ is the adjacency matrix with self-connections, $\mathbf{\tilde{D}}_{i} \in \mathbb{R}^{n_{i} \times n_{i}}$ is the diagonal degree matrix of $\mathbf{\tilde{A}}_{i} $, $\mathbf{X}_{i} \in \mathbb{R}^{n_{i} \times d}$ is $n$ input features with dimension $d$, and $\mathbf{\Theta}_{a t t} \in \mathbb{R}^{d \times 1}$ is the learnable parameters of pooling layer.  Based on the attention score, we select top $c<n_{i}$ nodes that have larger scores with keeping their origin edges unchanged.

To get fixed representation dimension for each graph, we need Read-Out operation to form each graph embedding vector into identical dimension.  Following Zhang et al. \cite{zhang2019hierarchical}, we use the concatenation of mean-pooling and max-pooling for each level of graph embeddings of ${G}_{i}$ as follows:

\begin{equation}
\mathbf{r}_{i}^{l}=\mathcal{R}\left(\mathbf{H}_{i}^{l}\right)=\sigma\left(\frac{1}{n_{i}^{l}} \sum_{p=1} \mathbf{H}_{i}^{l}(p,:) \| \max _{q=1}^{d} \mathbf{H}_{i}^{l}(:, q)\right),
\end{equation}
where $\mathbf{r}_{i}^{l} \in \mathbb{R}^{2 d}$ is the $l$-th layer embedding, $n_{i}^{l}$ is the node number in $l$-th layer, $\mathbf{H}_{i}^{l}$ denotes $l$-th layer hidden representation matrix , $\|$ is concatenation operation, $p$ and $q$ are row number and column number respectively, $d$ is feature dimension, and $\sigma$ is the activation function (e.g., Rectified Linear Unit, ReLU \cite{ReLU}).

Following the graph embedding backbone, we compute  the final graph embedding of  ${G}_{i}$ as 
\begin{equation}
\mathbf{z}_{i}= \mathbf{r}_{i}^{1}+\mathbf{r}_{i}^{2}+\cdots+\mathbf{r}_{i}^{L}
\end{equation}
 and put it into Multi-Layer Perceptron (MLP) classifier to perform classification using cross-entropy loss.
 \subsection{Meta-Learner for Fast Adaptation}
We use $\mathbf{\theta}_{e}$ and $\mathbf{\theta}_{c}$ to denote the parameters of graph embedding modules and MLP classifier respectively, where $\mathbf{\theta}_{e}$ contains the parameters of node embedding layers and pooling layers. 
To achieve the fast adaptation of $\mathbf{\theta}_{e}$ and $\mathbf{\theta}_{c}$, we put them into a nested loop framework to create a GNNs based meta-learner.  Specifically, our meta-learner is optimized from two procedures. One of the procedures is called the outer loop aiming to get optimal initialization for new classification tasks, and one is called the inner loop to implement fast adaptation based on a suitable initialization. Algorithm \ref{alg:algorithm1} elaborates on how to train a graph meta-learner at the training state. First, we sample support data $D^{train}_{sup}$ and query data $D^{train}_{que}$ in an episode. Then we perform adaptation operation by updating $\mathbf{\theta}_{e}$ and $\mathbf{\theta}_{c}$ for $T$ steps on $D^{train}_{sup}$. Lines 7 to 8 in Algorithm \ref{alg:algorithm1} demonstrate the adaptation of meta-learner. After adaptation, demonstrated by line 13, we use the losses on $D^{train}_{que}$ to perform backpropagation and update $\mathbf{\theta}_{e}$ as well as $\mathbf{\theta}_{c}$. Similarly, at the test stage, the meta-learner will perform adaptation on labeled support graphs $D^{test}_{sup}$ and predict the label of query graphs $D^{test}_{que}$. 

\begin{algorithm}[tb]
\caption{Training Stage of AS-MAML}
\label{alg:algorithm1}
\begin{flushleft}
\textbf{Input}: Task distribution $p(\mathcal{T})$ over $\{(G^{train},\mathbf{y}^{train})\}$\\
\textbf{Parameter}: Graph embedding parameters $\theta_{e}$, classifier parameters $\theta_{c}$, step controller parameters $\theta_{s} $, learning rate $ \alpha_{1}, \alpha_{2}, \alpha_{3}$ \\
\textbf{Output}: The trained parameters $\theta_{e},\theta_{c},\theta_{s}$

\end{flushleft}

\begin{algorithmic}[1] 
\STATE Randomly initialize $\theta_{e} $, $\theta_{c} $, $\theta_{s} $
\WHILE{not convergence}
\STATE Sample task $\mathcal{T}_{i}$ with support graphs $D^{train}_{sup}$ and query graphs $D^{train}_{que}$
\STATE Get adaptation step $T$ via  Equation \ref{calT}
\STATE Set fast adaptation parameters: $\theta^{\prime}=\theta=\{\theta_{e},\theta_{c}\}$
\FOR {$t = 0 \to T$ }
\STATE Evaluate  $\nabla_{\theta^{\prime}} \mathcal{L}_{\mathcal{T}_{i}}\left(f_{\theta^{\prime}}\right)$ on $D^{train}_{sup}$ by classification loss $\mathbf{L}^{(t)}$.
\STATE Update $\theta^{\prime} $ : $\theta^{\prime} \leftarrow \theta^{\prime}- \alpha_{1} \nabla_{\theta^{\prime}} \mathcal{L}_{\mathcal{T}_{i}}\left(f_{\theta^{\prime}}\right)$
\STATE Calculate ANI $\mathbf{M}^{(t)}$ via Equation \ref{eq:ANIs}
\STATE Calculate stop probability $p(t)$ via Equation \ref{stop_prb}
\STATE Calculate reward ${Q^{(t)}}$ on $D^{train}_{que}$  by Equation \ref{eq:Q}
\ENDFOR
\STATE  $\theta \leftarrow \theta-\alpha_{2} \nabla_{\theta}  \mathcal{L}_{\mathcal{T}_{i}}\left(f_{\theta^{\prime}}\right)$ on $D^{train}_{que}$
\FOR {$t = 0 \to T$ }
\STATE $\theta_{s} \leftarrow \theta_{s} +  \alpha_{3}  Q^{(t)} \nabla_{ \theta_{s}} \ln p(t)$
\ENDFOR
\ENDWHILE
\end{algorithmic}
\end{algorithm}
\subsection{Adaptation Controller}
Finding optimal combinations of learning rates and step size is difficult for MAML \cite{antoniou2018how}. Besides, arbitrary graph size and structure bring difficulty for ascertaining optimal step size manually. As an empirical solution to alleviate these problems, we design an RL based controller to decide optimal step size for the adaptation of meta-learner when given other hyper-parameters. Therefore, our controller must roughly know when to stop adaptation according to the embedding quality and training state (denoted by loss). We focus on Average Node Information (ANI) to indicate the embedding quality. Intuitively, if a node can be well reconstructed by its neighborhoods, it has less information for the graph classification. Similarly, the rising of batch graphs' ANI indicates that the pooling module has learned to select the most informative nodes.  Hou et al. \cite{Hou2020Measuring} proposed similar concept called \textit{Feature Smoothness} measuring node information over graphs.
Here we adopt another practical method defined by \cite{zhang2019hierarchical},  where they compute node information as the Manhattan distance between the node representation itself and the one constructed from its neighbors. Inspired by their work, we define the ANI of a single graph ${G}_{i}$ as follows:
    \begin{equation}
    \label{eq:ANI}
{ANI}_{i}^{l}= \frac{1}{n_{i}^{l}}\sum_{j=1} \left\| \left[ \left(\mathbf{I}_{i}^{l}-\left(\mathbf{D}_{i}^{l}\right)^{-1} \mathbf{A}_{i}^{l}\right) \mathbf{H}_{i}^{l}  \right]_{j} \right\|_{1},
\end{equation}
where $l$ denotes the embedding layer of the graph, $n_{i}^{l}$ denotes the number of node, $j$ denotes the row index of matrix or $j$-th node, ${\| \cdot \|}_{1}$ denotes the L1 norm of row vector,  $\mathbf{A}_{i}^{l}$ denotes the adjacency matrix, $\mathbf{D}_{i}^{l}$ is the degree matrix of $\mathbf{A}_{i}^{l}$, $\mathbf{H}_{i}^{l}$ denotes $l$-th layer hidden representation matrix. In our work, we only use the last layer of graph embedding (i.e., $\mathbf{H}_{i}^{L}$). And unless specifically stated, we use scalar value ANI to denote the average node information of the batch graphs:
 \begin{equation}
 \label{eq:ANIs}
ANI=1/n*\sum_{i=1}^{n} ANI_{i}^{L},
\end{equation}
where $n$ denotes the number of batch graphs, $L$ denotes the L-th layer of graph embedding.

Next we set the number of initial step as $T_{i}$, the ANIs in $T_{i}$ steps as $\mathbf{M} \in \mathbb{R}^{T_{i} \times 1}$ and denote classification losses as $\mathbf{L} \in \mathbb{R}^{T_{i} \times 1}$. Then we compute stop  probability $p^{(t)}$ at step $t$ as follows:
\begin{equation}
\label{stop_prb}
\boldsymbol{h}^{\left(t\right)}=\operatorname{LSTM}\left(\left[\mathbf{L}^{(t)},\mathbf{M}^{(t)}\right], \boldsymbol{h}^{(t-1)} \right),  \\
p^{(t)}=\sigma\left(\mathbf{W}\mathbf{h}^{(t)}+\mathbf{b}\right),
\end{equation}
where $\mathbf{W}$ and $\mathbf{b}$ are the parameters of a MLP module, $\sigma$ is sigmoid function and $h^{(t)}$ is the output of LSTM module. Note that the adaptation will not stop until $T_{i}$ steps at current task regardless of $p^{(t)}$. We set $p^{(t)}$ as a prior for the next task:
\begin{equation}
\label{calT}
T_{i+1} =  \left \lfloor \frac{1}{p^{(T_{i})}} \right \rfloor ,
\end{equation}where $ \lfloor \rfloor$ is round down operation. We can observe that we compute $T$ of next task according to the ANIs and losses, where both of them are produced by current task. The reason behind it is that if we stop adaptation according to $p^{(t)}$ in current task, it will output  larger variance of step and bring instability for optimization.  
Besides, we get controller's rewards as following:

\begin{equation}
\label{eq:Q}
Q^{(t)}=\sum_{t=1}^{T} r^{(t)}=\sum_{t=1}^{T}(e_{T}-e_{t} - \eta*t),
\end{equation}
where $T$ is total steps and $e_{t}$ is the classification accuracy on query data at step $t$, and $\eta*t$ denotes the penalty item.  Then we update our controller by policy gradients, which is a typical method in RL: 
\begin{equation}
\theta_{s} = \theta_{s} +  \alpha_{3}  Q^{(t)}\nabla_{ \theta_{s}} \ln p(t),
\end{equation}where $\ln p(t)$ is the log of  stop probability p(t) $\cdot$, $\nabla_{ \theta_{s}}$ is the gradients over $\theta_{s}$ and $\alpha_{3}$ is learning rate. 

 \section{Graph structure-dependent bound for meta-learning}
 \label{section:bound}
In this section, we theoretically analyze the proposed framework with 
unnecessary details omitted. Driven by graph data, we will build a generalization error bound in meta-learning scenario considering the locality of graph structure (e.g., the degrees of nodes in a graph). Before getting closer to analysis, we first give the key results: (1) The bound is dependent on the number of graphs in the support example of test classes and the number of graphs of training classes.  (2) The bound is dependent on the locality structure of graphs between training classes and test classes. The first result is common in meta-learning scenarios, while the second result is derived from the structure of graphs.

From the perspective of representation learning and probability distribution, our framework try to minimize the distance of latent distributions between $\mathcal{G}^{train} $ and $\mathcal {G}^{test} $ directly.  Based on this premise, we use the integral probability metric (IPM, \cite{IPM}) as a general distance metric to produce the upper bound. IPM has been employed in the analysis of transfer learning \cite{10.5555/2999325.2999503}, non-parametric estimation \cite{NPE} and generative models \cite{zhang2018on}. Generally, IPM can be formalized as:

\begin{equation}
\gamma_{\mathcal{H}}(\mathbb{P}, \mathbb{Q}):=\sup _{h \in \mathcal{H}}\left|\mathbb{E}_{\mathbb{P}} h(z)-\mathbb{E}_{\mathbb{Q}} h(z)\right|,
\end{equation}
where $\mathbb{P}$ and $\mathbb{Q}$ denote two different probability distributions, respectively,  $\mathcal{H}$ denotes the collection of real-valued functions (e.g., square loss function and margin loss function ). 

We set the empirical distribution of $\mathcal{G}^{train} $ as  $\hat{\mathbb{P}}$.  With a slight abuse of notation, in a single task at the test stage, we set the empirical distribution of support graph $D^{test}_{sup}$ as $\hat{\mathbb{Q}}$, and set the expected distribution query data $D^{test}_{que}$ as $\mathbb{Q}$. Then in the adaptation process at test stage, we actually want to  minimize the following bound:

\begin{equation}
\gamma_{\mathcal{H}}(\hat{\mathbb{P}}, {\mathbb{Q}}) = \sup _{h \in \mathcal{H}}\left|\mathbb{E}_{\hat{\mathbb{P}}} h({G})-\mathbb{E}_{{\mathbb{Q}}} h({G})\right|,
\label{equation:PhatQ}
\end{equation}
where $h(G)$ is the classification loss of meta-learner after the adaptation on $D^{test}_{sup}$.  
Then with the help of IPM, $\gamma_{\mathcal{H}}(\hat{\mathbb{P}}, {\mathbb{Q}}) $can be bounded by following theorem \cite{cai2020weighted}:
\begin{theorem}
Let $\mathcal{H}$ denote a class of functions whose members map from ${G}_{i}$ to [a, b], and suppose that the training data test data are independent, and that the data instances of each are i.i.d. within a sample. Let $\epsilon > 0 $. Then with probability at least $1 - \epsilon$ over the draws of the training and query samples, 
\begin{equation}
\gamma_{\mathcal{H}}\left(\hat{\mathbb{P}}, \mathbb{Q}\right) \leq \gamma_{\mathcal{H}}\left(\hat{\mathbb{P}}, \hat{\mathbb{Q}}\right)+2 \mathcal{R}\left( \mathcal{H} | \{{G}_{i}\}^{m}_{i=1} \right)+3 \sqrt{\frac{(b-a)^{2} \log (2 / \epsilon)}{2m}},
\end{equation}
where $m$ is the number of support graphs $D^{test}_{sup}$,  $\mathcal{R}\left( \mathcal{H} |\{{G}_{i}\}^{m}_{i=1} \right)$ denotes the empirical Rademacher complexity \cite{Radermacher_com} of the function class $\mathcal{H}$  w.r.t. the support graphs.
\label{theorem1}
\end{theorem}

Now we provide the proof of Theorem \ref{theorem1} referencing Cai et al. \cite{cai2020weighted}. First, IPM is interrelated with  \textit{uniform deviation with empirical Rademacher complexity} \cite{Radermacher_com}:
\begin{theorem}
Let $\mathcal{H}$ denote a class of functions whose members map from ${z}_{i}$ to [a, b], and suppose that $\{{z}_{i}\}^{m}_{i=1}$ is sampled from a i.i.d distribution $\mathbb{P}$. Then for $\epsilon > 0 $. Then with probability at least $1 - \epsilon$ over the sample, 
\begin{equation}
\sup _{h \in \mathcal{H}}\left|\mathbb{E}_{\mathbb{\hat{P}}} h(z)-\mathbb{E}_{\mathbb{P}} h(z)\right| \leq 2 \mathcal{R}\left( \mathcal{H} | \{{z}_{i}\}^{m}_{i=1} \right)+3 \sqrt{\frac{(b-a)^{2} \log (2 / \epsilon)}{2m}},
\label{theorem2}
\end{equation}where $\mathbb{\hat{P}}$  represents the empirical distribution of the sample, $ \mathcal{R}\left( \mathcal{H} | \{{z}_{i}\}^{m}_{i=1} \right)$ denotes the empirical Rademacher complexity of the function class $\mathcal{H}$  w.r.t. the sample. 
\end{theorem}

\begin{proof}
Then Equation \ref{equation:PhatQ} is unfolded as:
\begin{equation}
\begin{split}
\gamma_{\mathcal{H}}(\hat{\mathbb{P}}, {\mathbb{Q}}) &= \sup _{h \in \mathcal{H}}\left|\mathbb{E}_{\hat{\mathbb{P}}} h({G})  -\mathbb{E}_{{\mathbb{Q}}} h({G})\right|  \\
&= \sup _{h \in \mathcal{H}}\left|\mathbb{E}_{\hat{\mathbb{P}}} h({G}) + \mathbb{E}_{{\hat{\mathbb{Q}}}} h({G}) -\mathbb{E}_{\hat{\mathbb{Q}}} h({G}) -\mathbb{E}_{{\mathbb{Q}}} h({G})\right| \\
& \leq \sup _{h \in \mathcal{H}}\left[ \left|\mathbb{E}_{\hat{\mathbb{P}}} h({G}) - \mathbb{E}_{{\hat{\mathbb{Q}}}} h({G}) \right| + \left|
\mathbb{E}_{\hat{\mathbb{Q}}} h({G}) -
\mathbb{E}_{{\mathbb{Q}}} h({G}) \right| \right] \\
& \leq \sup _{h \in \mathcal{H}} \left|\mathbb{E}_{\hat{\mathbb{P}}} h({G}) - \mathbb{E}_{{\hat{\mathbb{Q}}}} h({G}) \right| + 
\sup _{h \in \mathcal{H}} \left| \mathbb{E}_{\hat{\mathbb{Q}}} h({G}) -
\mathbb{E}_{{\mathbb{Q}}} h({G}) \right|  \\
&=\gamma_{\mathcal{H}}\left(\hat{\mathbb{P}}, \hat{\mathbb{Q}}\right) + \sup _{h \in \mathcal{H}} \left| \mathbb{E}_{\hat{\mathbb{Q}}} h({G}) -
\mathbb{E}_{{\mathbb{Q}}} h({G}) \right|
\end{split}
\label{equation:proof}
\end{equation}When the last term of Formula \ref{equation:proof} substituted by Theorem \ref{theorem2},  Theorem \ref{theorem1} is proved.
\end{proof}

After giving the proof of Theorem \ref{theorem1}, $\gamma_{\mathcal{H}}\left(\hat{\mathbb{P}}, \hat{\mathbb{Q}}\right)$ can be optimized by the adaptation of meta-learner on the test stage.  As the final step, we focus on the evaluation of $\mathcal{R}\left( \mathcal{H} |\{{G}_{i}\}^{m}_{i=1} \right)$. The diversity of graph structure brings difficulty for directly calculating this term. We adopt the method proposed by K. Garg et al. \cite{garg2020generalization},  who gave an upper bound for $\mathcal{R}\left( \mathcal{H} |\{{G}_{i}\}^{m}_{i=1} \right)$, by transforming each graph to the corresponding collection of local computation trees. Because when updating every node embedding, the node to be updated can be viewed as a root node of a tree, with its neighbor nodes as children.  Given this hypothesis, they derived a strict upper bound mainly based on the degree of nodes, where the degree was used to denote the local complexity of graph. Furthermore, the local complexity of graph is made full use by our model. In the experiment of Section \ref{comparisonGSM}, our model performs better over datasets which have 
clear local structure.

\section{Experiments}
In the experiments, we focus on two aspects: (1)How does the framework perform on few-shot graph classification tasks? (2) How does the controller work when training meta-learner? In this section, we will introduce experiment datasets, comparison with baselines and details of implementation. Finally, we demonstrate the effectiveness of key modules by ablation study and detail analysis.
\subsection{Datasets}
We select four public graph datasets including COIL-DEL, R52, Letter-High and  TRIANGLES. These datasets are publicly available \footnote{https://ls11-www.cs.tu-dortmund.de/staff/morris/graphkerneldatasets} \footnote{https://www.cs.umb.edu/~smimarog/textmining/datasets/}. The statistics are summarized in Table \ref{tab:1}. The visualization of each datasets are shown in Figure \ref{fig:data}. In the work proposed by \cite{Chauhan2020FEW-SHOT}  for few-shot graph classification, they  focus on two datasets containing Letter-High and  TRIANGLES. We use two additional datasets where Graph-R52 was built from text classification dataset and  COIL-DEL was built from images. 
\paragraph{COIL-DEL} COIL-DEL is built on images, and each graph is constructed by applying corner detection and Delaunay triangulation to corresponding image \cite{COIL}. 
\paragraph{R52} R52 is a text dataset in which each text is viewed as a graph. We transformed it into a graph dataset as follows: if two words appear together in a specified sliding window, they have an undirected edge in the graph. We keep classes with more than 20 samples and finally get 28 classes. We name the new dataset as Graph-R52 for clarity.
\paragraph{Letter-High} Each graph represents distorted letter prototype drawings with representing lines by undirected edges and ending points of lines by nodes \cite{COIL}. More specifically, Letter-High contains 15 classes from English alphabets: A, E, F, H, I, K, L, M, N, T, V, W, X, Y, Z. 
\paragraph{TRIANGLES} The dataset was proposed for the task of triangle counting, where the model is required to give the number of triangles of each graph. TRIANGLES contains 10 different graph classes numbered from 1 to 10 corresponding to the number of triangles in graphs of each class. In our experiments, we use the partition of \cite{Chauhan2020FEW-SHOT},  where they remove oversize graph samples so the total sample size of TRIANGLES is reduced from 45000 to 2000.
\begin{table}[!tbp]
\centering
\begin{tabular}{ccccccc}  
\toprule
Datasets   & $|G|$     & Avg.$|\mathcal{V}|$     & Avg.$|\mathcal{E}|$ & $C_{0}$ & $C_{1}$ & $C_{2}$ \\
\midrule
COIL-DEL   & 3900  & 21.54  & 54.24  & 60  & 16  & 20    \\
Graph-R52        & 8214   & 30.92    & 165.78    & 18  & 5  & 5   \\
Letter-High      & 2250   & 4.67     & 4.50      & 11  & 0  & 4   \\
TRIANGLES        & 45000  & 20.85    & 35.50     & 7   & 0  & 3   \\
\bottomrule
\end{tabular}
\caption{Statistics of datasets. For each dataset, we show total graph number $|G|$, average node number Avg.$|\mathcal{V}|$, Average edge number Avg.$|\mathcal{E}|$ and class number for training ($C_{0}$), validation ($C_{1}$) and test ($C_{2}$). }
\label{tab:1}
\end{table}
\begin{figure}[!t]\small
\centering
\subfigure[\scriptsize{COIL-DEL}]{
\label{fig:example1}
\centering
\includegraphics[width=0.22\textwidth]{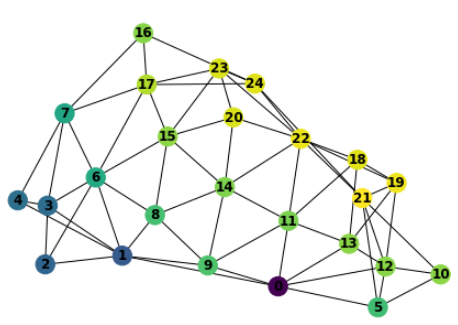}
}
\subfigure[\scriptsize{Graph-R52}]{
\label{fig:example2}
\centering
\includegraphics[width=0.22\textwidth]{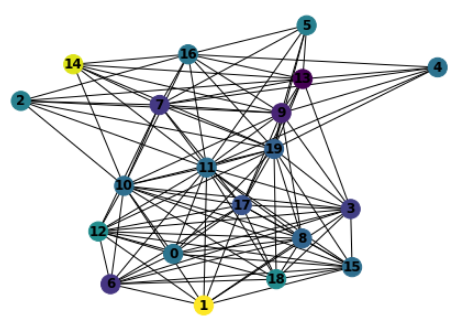}
}
\subfigure[\scriptsize{TRIANGLES}]{
\label{fig:example3}
\centering
\includegraphics[width=0.22\textwidth]{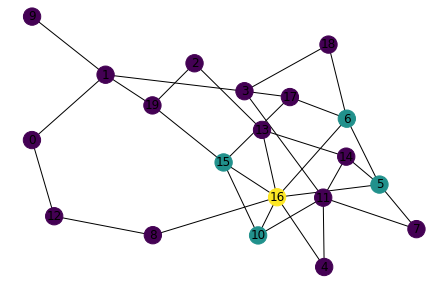}
}
\subfigure[\scriptsize{Letter-High}]{
\label{fig:example4}
\centering
\includegraphics[width=0.22\textwidth]{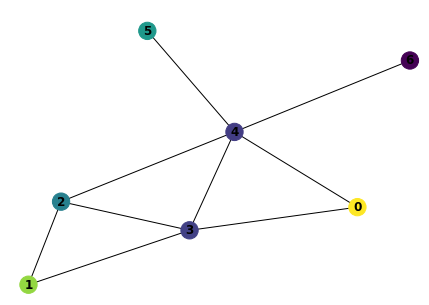}
}
\caption{Visualization of typical instances of the used datasets, where their different graph size and graph structures bring challenge for graph classification models.}
\label{fig:data}
\end{figure} 


\begin{table*}[!htbp]
\centering
\begin{tabular}{llcccc} 
\toprule
\multicolumn{1}{c}{\multirow{2}{*}{Categories}}  & \multicolumn{1}{c}{\multirow{2}{*}{Baselines}} & \multicolumn{2}{c}{\multirow{1}{*}{COIL-DEL}} & \multicolumn{2}{c}{\multirow{1}{*}{Graph-R52}} \\ 
\cline{3-6}
 & &5-way-5-shot &5-way-10-shot &2-way-5-shot &2-way-10-shot \\
\toprule[0.5pt]
\multirow{3}{*}{Kernels} & GRAPHLET & $47.47\pm1.06$ & $49.04\pm0.98$ & $56.52\pm1.46$ & $57.16\pm1.47$ \\
                         & SP        & $38.33\pm0.62$ & $42.18\pm0.69$ & $74.38\pm1.50$ & \textcolor{blue}{76.96 $\pm$ 1.34} \\
                         & WL        & $43.05\pm1.25$ & $52.28\pm1.47$ & \textbf{76.90 $\pm$ 1.48} & $75.91\pm1.46$ \\
\midrule

\multirow{1}{*}{Finetuning} & finetuning & $68.21\pm1.29$ & $72.38\pm1.40$ & $71.87\pm2.04$ & $72.39\pm1.88$ \\
\midrule
\multirow{9}{*}{GNNs-Pro} 
                        & GCN, TopKPool & $78.01\pm1.83$ & $78.98\pm1.53$ & $69.98\pm1.53$ & $70.19\pm1.37$ \\
                         & GCN, EdgePool     & $76.21\pm1.54$ & $79.43\pm1.58$ & $67.24\pm1.34$ & $67.72\pm1.59$ \\
                         & GCN, SAGPool          & $76.58\pm1.19$ & $79.16\pm1.06$ & $69.88\pm1.40$ & $70.46\pm1.47$ \\
                         & GraphSAGE, TopKPool    & $69.80\pm1.25$ & $74.18\pm1.73$ & $70.43\pm1.76$ & $70.52\pm1.83$ \\
                         
                         & GraphSAGE, EdgePool        & $80.08\pm1.26$ & $80.96\pm1.26$ & $68.13\pm1.59$ & $70.72\pm1.58$ \\
                         
                         & GraphSAGE, SAGPool & $79.30\pm1.12$ & $80.91\pm1.62$ & $68.10\pm1.40$ & $70.49\pm1.32$ \\
                         
                         & GAT, TopKPool          & $76.37\pm1.10$ & $77.29\pm1.40$ & $71.99\pm1.51$ & $73.31\pm1.44$ \\
                         
                         & GAT, EdgePool        & \textcolor{blue}{81.00 $\pm$ 1.22} & \textcolor{blue}{83.57 $\pm$ 0.99} & $66.49\pm1.32$ & $70.49\pm1.17$ \\
    
                         & GAT, SAGPool          & $72.54\pm1.07$ & $73.99\pm1.00$ & $67.78\pm1.52$ & $74.10\pm1.57$ \\
                        
\midrule

\multirow{2}{*}{Ours}   &AS-MAML (wo/AS)  & $79.54 \pm 1.48$ & $81.24\pm1.27$ & $74.12\pm1.39$ & $76.05\pm1.17$ \\
                        & AS-MAML (w/AS)   & \textbf{81.55 $\pm$ 1.39} & \textbf{84.75 $\pm$ 1.30} & \textcolor{blue}{75.33 $\pm$ 1.19} & \textbf{78.33 $\pm$ 1.17} \\
                                  
\bottomrule
\end{tabular}
\caption{Accuracies with a standard deviation of baseline methods and our framework. We tested 200 and 500 N-way-K-shot tasks on COIL-DEL and  Graph-R52, respectively. The \textbf{bold black} numbers denote the best results we get, and the  \textcolor{blue} {blue} numbers denote the second best results. AS-MAML (wo/AS) denotes our framework without Adaptive Step (AS) which is controlled by our step controller, and AS-MAML (w/AS) denotes the whole framework we proposed.}
\label{tab:2}
\end{table*}

\subsection{Baselines}
We adopt four groups of baselines made up of  Graph Kernel, Finetuning, GNNs-Prototypical-Classifier (GNNs-Pro) and Graph Spectral Measures (GSM) \cite{Chauhan2020FEW-SHOT} . For Graph Kernel baselines, we perform  N-way-K-shot graph classification over the test set directly because there are no parameters to transfer. The baselines of the last three groups train a GNNs based graph classifier by performing classification over $C_0$ training classes (see Table \ref{tab:1}). On the test stage, they perform N-way-K-Shot classification. 

\paragraph{Graph Kernel.} This group of methods firstly measure the similarity between labeled support data and query data on the test stage. After that, the similarity matrix was put into a Prototypical Classifier, which has none of parameters \cite{SnellSZ17}, to get predicted labels of query data. We choose typical graph kernel algorithms including Shortest-path  Kernel (SP) \cite{borgwardt2005shortest}, Graphlet Kernel \cite{shervashidze2009efficient} and Weisfeiler-Lehman Kernel (WL)  \cite{shervashidze2011weisfeiler}.

\paragraph{Finetuning.} In this baseline, we train a naive graph classifier consisting of GraphSAGE, SAGPool and MLP classifier. On the test stage, we change the output dimension of the last layer of the classifier and fine-tuning the layer's parameters, while keeping other modules unchanged. 

\paragraph{GNNs-Pro.}  We train a graph classifier following Finetuning. On the test stage, we replace the MLP classifier with Prototypical Classifier. We choose GCN \cite{kipf2016semi}, GraphSAGE \cite{SAGE} and GAT \cite{GAT} as graph convolutional modules, and  Self-attention Pooling (SAGPool) \cite{lee2019self}, TopK Pooling (TopKPool) \cite{unet} and Edge Pooling (EdgePool) \cite{diehl2019edge} as graph pooling modules. 


\paragraph{GSM} Chauhan et al. \cite{Chauhan2020FEW-SHOT} proposed the GSM based method customized for few-shot graph classification. On the training stage, they compute prototype graphs from each class, then they cluster the prototype graphs to produce super-classes. After that, they predict the origin class and super-class of each graph. On the test stage, they only update the classifier based on the classification of origin classes.

\subsection{Experimental Details}
To ensure a fair comparison, we use three convolutional layers followed by corresponding pooling layers for the GNNs based baselines and our proposed framework. We set the same node dimension as 128 for all GNNs based baselines. For the adaptation step, we set the minimum and maximum step by 4 and 15. We implement GNNs based baselines and our framework with  PyTorch Geometric (PyG \footnote{ https://github.com/rusty1s/pytorch\_geometric } )  and graph kernel baselines based on GraKel {\footnote{https://github.com/ysig/GraKeL}}.  We use SGD optimizer with 1e-5 for weight decay and versatile learning rates 0.0001, 0.001,  0.0001 for $\alpha_{1}$, $\alpha_{2}$ and $\alpha_{3}$, respectively.

\subsection{Comparison with Graph Kernel, Finetuning and GNNs-Pro}
 \label{section:exp}
To evaluate the performance of our framework, we performed 5-way-5-shot and 5-way-10-shot graph classification on COIL-DEL dataset. On Graph-R52 dataset, we performed 2-way-5-shot and 2-way-10-shot graph classification.  The results are reported in Table \ref{tab:2}.  
 Our framework utilizes GraphSAGE and SAGPool as graph embedding backbone. So firstly we compare our framework with the finetuning baseline built on GraphSAGE and SAGPool. We found our framework is superior to the finetuning baseline with a large margin, which indicates that the meta-leaner works well with a fast adaptation mechanism.  Moreover, under 5-way-10-shot setting in GraphSAGE-SAGPool baseline, our framework achieves about 3.84\%  improvement on the COIL-DEL dataset.

Surprisingly, traditional graph kernel based baselines achieve competitive performance on Graph-R52 dataset. The reasons are two-fold: (1)  our graphs from texts contain many well defined sub-graphs built by text themes and their neighbor words, and this pattern gives graph kernels a favorable position; (2) the parameters of kernel methods are far less than GNNs based methods. So they are not prone to be overfitting, which is GNNs' common problem for the few-shot task. However,  finding an appropriate kernel is difficult (e.g., From Table \ref{tab:2}, SP behaves badly compared to WL and GRAPHLET on the COIL-DEL dataset).

\begin{figure*}[!t]\small
\centering
\subfigure[\scriptsize{GSM on TRIANGLES}]{
\label{fig:exa1}
\centering
\includegraphics[width=0.21\textwidth]{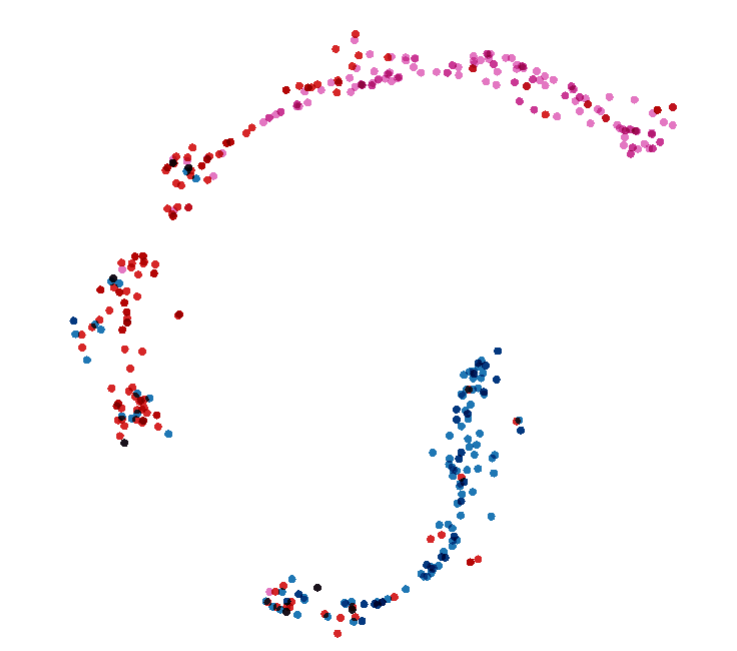}
}
\subfigure[\scriptsize{ AS-MAML on TRIANGLES}]{
\label{fig:exa}
\centering
\includegraphics[width=0.21\textwidth]{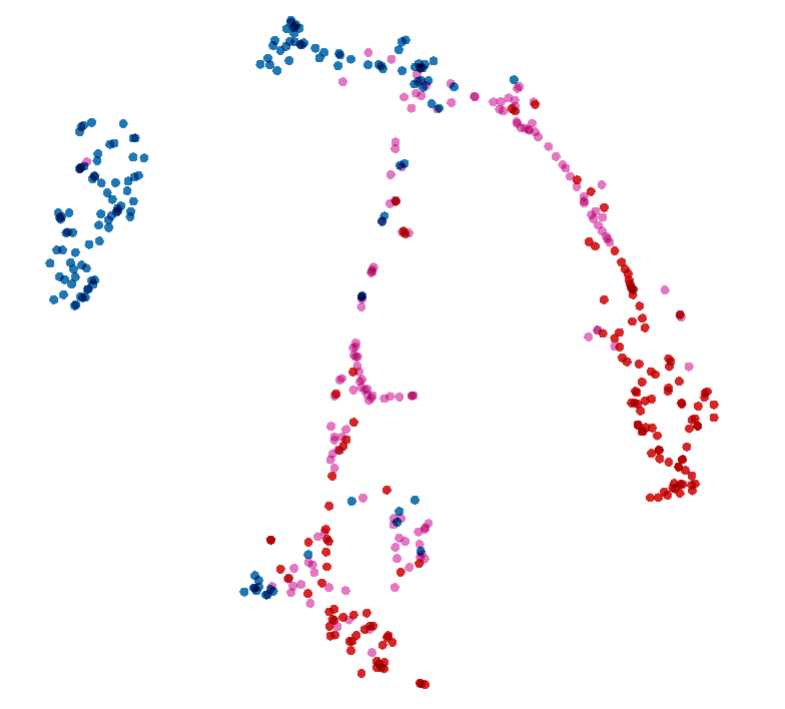}
}
\subfigure[\scriptsize{GSM on Letter-High}]{
\label{fig:exm3}
\centering
\includegraphics[width=0.21\textwidth]{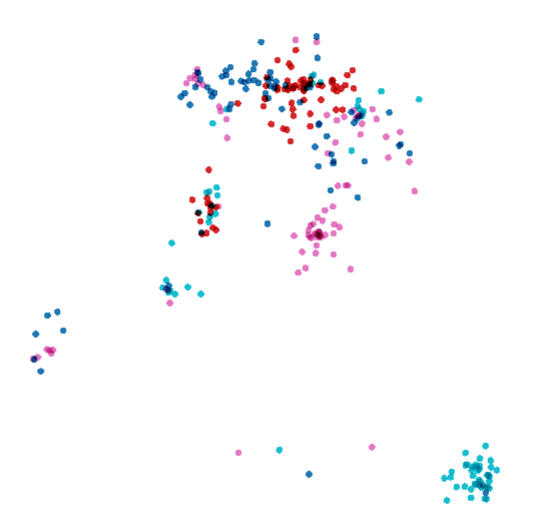}
}
\subfigure[\scriptsize{AS-MAML on Letter-High}]{
\label{fig:exa4}
\centering
\includegraphics[width=0.21\textwidth]{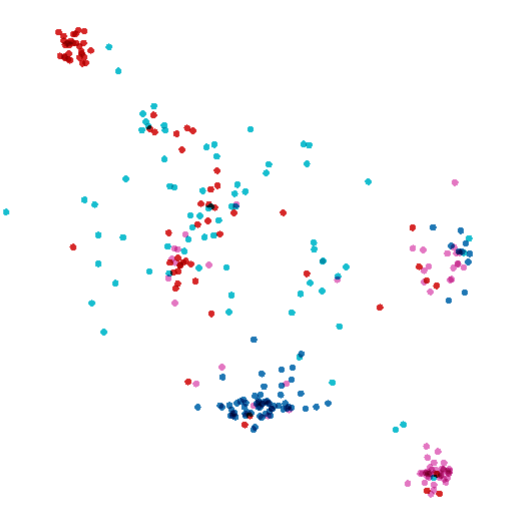}
}
\caption{T-SNE visualization of unseen classes under 5-way-5-shot scenarios. We randomly sample 400 examples per dataset and visualized the representations before they were put into classifying layer. We can see more clear distribution boundary using our method.}
\label{fig:tsne}
\end{figure*} 
\begin{table}[!tpbh]
\centering
\begin{tabular}{cccc} 
\toprule
\multicolumn{1}{c}{\multirow{2}{*}{Methods}} & \multicolumn{1}{c}{\multirow{2}{*}{Shots}} & \multicolumn{2}{c}{\multirow{1}{*}{Datasets}} \\ 

\cline{3-4}
& &TRIANGLES &Letter-High  \\ 
\toprule[0.6pt]

\multirow{2}{*}{GSM}   &5-shot  &  $71.40\pm4.34$ & $69.91\pm5.90$   \\

                       &10-shot &  $75.60\pm3.67$ & $73.28\pm3.46$  \\
\midrule                      
\multirow{2}{*}{Ours}  &5-shot  &  $86.47\pm0.74$ & $76.29\pm0.89$   \\

                       &10-shot &  $87.26\pm0.69$ & $77.87\pm0.75$  \\
\bottomrule
\end{tabular}
\caption{Accuracies evaluated from GSM and AS-MAML we proposed. For AS-MAML, we test 200 N-way-K-shot tasks on both datasets. For GSM, we use the best results in their paper.}
\label{tab:3}
\end{table}

\begin{figure*}[!htp]
\centering
\subfigure[\scriptsize{Test accuracies}]{
\label{fig:com1}
\centering
\includegraphics[width=0.22\textwidth]{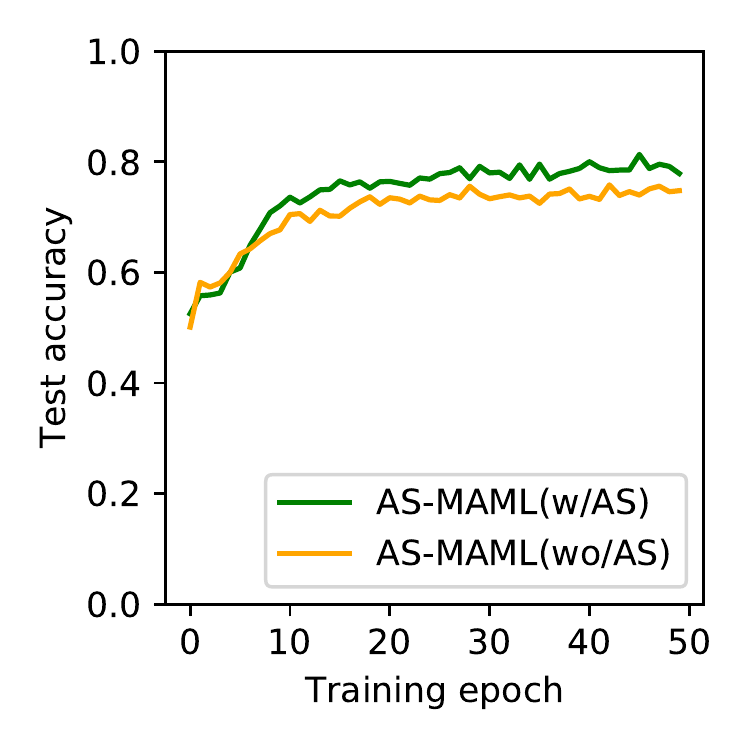}
}
\subfigure[\scriptsize{ ANI \& training accuracies}]{
\label{fig:com2}
\centering
\includegraphics[width=0.22\textwidth]{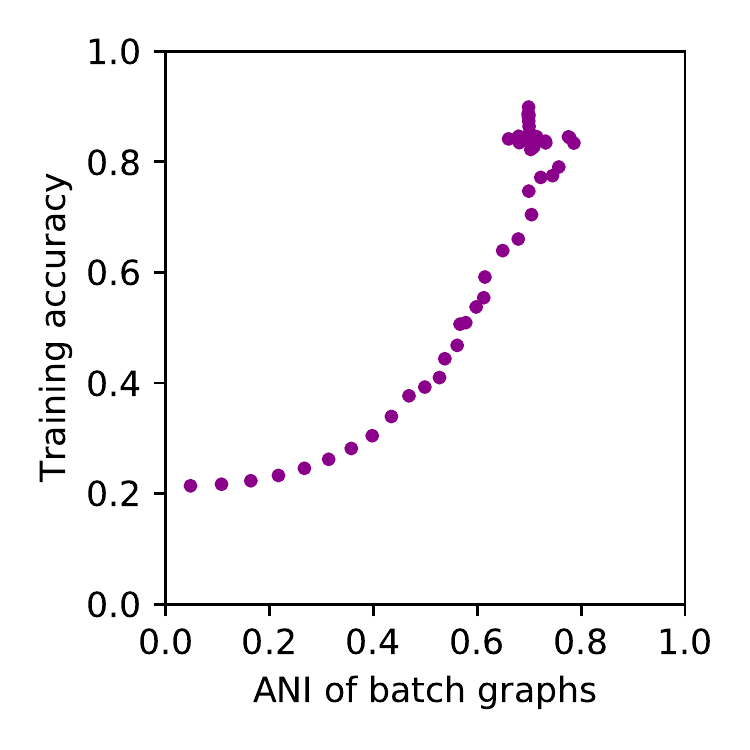}
}
\subfigure[\scriptsize{Adaptation steps }]{
\label{fig:com3}
\centering
\includegraphics[width=0.22\textwidth]{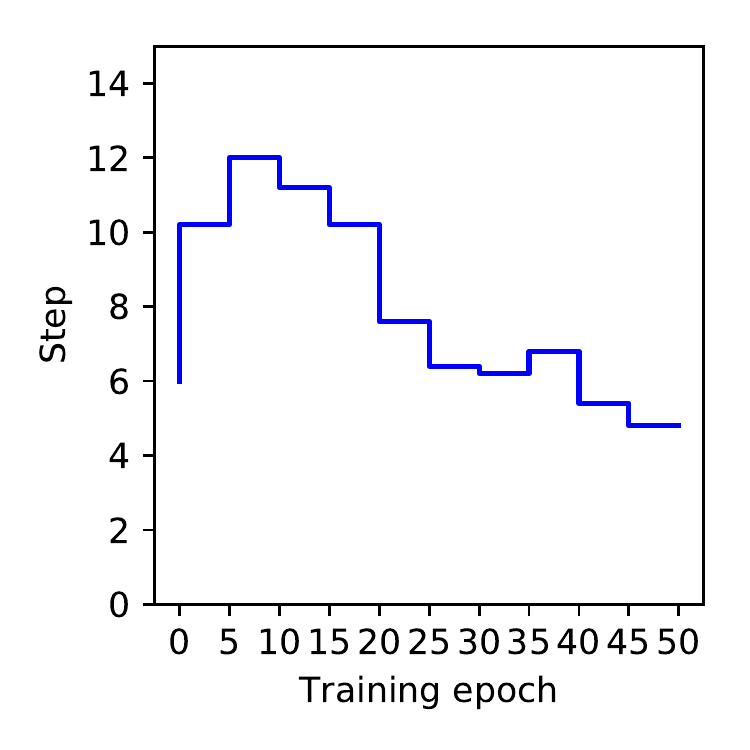}
}

%
\caption{Illustrations of the learning process under the 5-way-10-shot setting on  COIL-DEL dataset. (a) Test accuracies calculated from AS-MAML (wo/AS) and AS-MAML (w/AS) respectively. The adaptation step of AS-MAML (wo/AS) is 6, and other hyper-parameters are same as AS-MAML (w/AS). The initial values are reported after $0$-th training epoch. (b) The normalized ANIs and training accuracies in the first 50 epochs. Both of them are extracted from support graphs of the training set. (c) The variations of the adaptation step on the training stage. the value at epoch 0 is the initial adaptation step, and then we calculate an average for every 5 epochs. }
\label{fig:train_process}
\end{figure*} 

\subsection{Comparison with GSM Based Method}
\label{comparisonGSM}
The GSM based method proposed by Chauhan et al. \cite{Chauhan2020FEW-SHOT} did not adopt the episodic training paradigm, which is a key idea in our paper, so the method is inappropriate to be trained by N-way-K-shot graph classification on COIL-DEL and Graph-R52 dataset. For a fair comparison, we evaluate our framework on TRIANGLES and Letter-High, which are typical datasets used in their paper. As their partition, we randomly split out 20\% examples from the training set to perform validation. Following their test configuration, we perform 3-way-K-shot and 4-way-K-shot classification on TRIANGLES and Letter-High respectively. The comparison of performances is shown in Table \ref{tab:3}. From the table, we conclude that our framework outperforms theirs with a large margin. The reason behind it is that they assume the test classes belong to the same set of super-classes built from the training classes. However, the label spaces of training classes and test classes usually do not overlap in few-shot settings. We observe that the graphs of training classes and test classes have similar sub-structures, which can be discovered by a well initialized meta-learner within a few adaptation steps. As mentioned before, different classes in TRIANGLES have similar triangle structure. Therefore our framework gets the most obvious improvement on this dataset. Besides, the visualization (Figure \ref{fig:tsne}) of latent representation on unseen classes demonstrates the effectiveness of our method.

\subsection{Ablation Study and Detail Analysis} 
In this section, we show the effect of the controller module by ablation study. First of all,  without the adaptive step (AS), we evaluate the performance of our framework by just putting GraphSAGE, SAGPool into meta-learner. From Table \ref{tab:2}, we found that under 2-way-10-shot setting on Graph-R52, AS-MAML (wo/AS) brings about 3.06\% improvement compared with finetuning baseline. Furthermore, the step controller brings about 2.28\% improvement under 2-way-10-shot setting on Graph-R52 and 3.51\% improvement under 5-way-5-shot setting on COIL-DEL.

We did a deeper analysis for ANI and give more details of the step controller module under 5-way-10-shot setting on  COIL-DEL dataset. 
Figure \ref{fig:com1} shows the effect on the test set after adding the adaptive step (AS).
The scatter diagram (Figure \ref{fig:com2}) shows that ANIs have a positive correlation with classification accuracies, which means larger ANI indicates better graph embedding for the MLP classifier module. The advantage of ANI against classification accuracy is that a larger ANI implies more discriminative graph embedding modules, while a better classification accuracy may mean that the MLP classifier module is overfitted on poor graph embedding modules. Finally, Figure \ref{fig:com3} shows the variations of the adaptation step produced by the controller. At the beginning, the controller receives larger loss and smaller ANI, so it gives more adaptation steps to meta-learner for encouraging exploration. When the meta-learner has been trained well, the controller receives smaller loss and larger ANI, so it outputs smaller step size to alleviate overfitting. 

\section{Conclusion and Future Works}
Modeling real-world data into graphs is getting more attention in recent years. In this paper, we focus on few-shot graph classification and propose a novel framework named AS-MAML. To control the meta-learner's adaptation step, we proposed a novel step controller in a RL way by using ANI to demonstrate embedding quality. Beyond that, ANI is calculated by unsupervised way like estimating Mutual Information on graphs \cite{2018deep}. Exploring and utilizing them for graph representation learning is an interesting future work.  Moreover, we expect better graph embedding methods to improve the performance of our framework, including GIN and its variants. We also expect our work can be expanded to more challenging graph classification tasks like skeleton based action recognition, protein classification and subgraph analysis of social networks.

\bibliographystyle{ACM-Reference-Format}
\bibliography{main}


\begin{thebibliography}{59}


\ifx \showCODEN    \undefined \def \showCODEN     #1{\unskip}     \fi
\ifx \showDOI      \undefined \def \showDOI       #1{#1}\fi
\ifx \showISBNx    \undefined \def \showISBNx     #1{\unskip}     \fi
\ifx \showISBNxiii \undefined \def \showISBNxiii  #1{\unskip}     \fi
\ifx \showISSN     \undefined \def \showISSN      #1{\unskip}     \fi
\ifx \showLCCN     \undefined \def \showLCCN      #1{\unskip}     \fi
\ifx \shownote     \undefined \def \shownote      #1{#1}          \fi
\ifx \showarticletitle \undefined \def \showarticletitle #1{#1}   \fi
\ifx \showURL      \undefined \def \showURL       {\relax}        \fi
\providecommand\bibfield[2]{#2}
\providecommand\bibinfo[2]{#2}
\providecommand\natexlab[1]{#1}
\providecommand\showeprint[2][]{arXiv:#2}

\bibitem[\protect\citeauthoryear{Antoniou, Edwards, and Storkey}{Antoniou
  et~al\mbox{.}}{2019}]%
        {antoniou2018how}
\bibfield{author}{\bibinfo{person}{Antreas Antoniou}, \bibinfo{person}{Harrison
  Edwards}, {and} \bibinfo{person}{Amos Storkey}.}
  \bibinfo{year}{2019}\natexlab{}.
\newblock \showarticletitle{How to train your {MAML}}. In
  \bibinfo{booktitle}{\emph{ICLR}}.
\newblock


\bibitem[\protect\citeauthoryear{Bartlett and Mendelson}{Bartlett and
  Mendelson}{2003}]%
        {Radermacher_com}
\bibfield{author}{\bibinfo{person}{Peter~L. Bartlett} {and}
  \bibinfo{person}{Shahar Mendelson}.} \bibinfo{year}{2003}\natexlab{}.
\newblock \showarticletitle{Rademacher and Gaussian Complexities: Risk Bounds
  and Structural Results}.
\newblock \bibinfo{journal}{\emph{J. Mach. Learn. Res.}} \bibinfo{volume}{3},
  \bibinfo{number}{null} (\bibinfo{date}{March} \bibinfo{year}{2003}),
  \bibinfo{pages}{463–482}.
\newblock
\showISSN{1532-4435}


\bibitem[\protect\citeauthoryear{Baxter}{Baxter}{2000}]%
        {BaxterBiasLearning}
\bibfield{author}{\bibinfo{person}{Jonathan Baxter}.}
  \bibinfo{year}{2000}\natexlab{}.
\newblock \showarticletitle{A model of inductive bias learning}.
\newblock \bibinfo{journal}{\emph{Journal of Artificial Intelligence Research}}
  (\bibinfo{year}{2000}).
\newblock


\bibitem[\protect\citeauthoryear{Behl, Baydin, and Torr}{Behl
  et~al\mbox{.}}{2019}]%
        {behl2019alpha}
\bibfield{author}{\bibinfo{person}{Harkirat~Singh Behl},
  \bibinfo{person}{Atılım~Güneş Baydin}, {and} \bibinfo{person}{Philip
  H.~S. Torr}.} \bibinfo{year}{2019}\natexlab{}.
\newblock \bibinfo{title}{Alpha MAML: Adaptive Model-Agnostic Meta-Learning}.
\newblock
\newblock
\showeprint[arxiv]{cs.LG/1905.07435}


\bibitem[\protect\citeauthoryear{Borgwardt and Kriegel}{Borgwardt and
  Kriegel}{2005}]%
        {borgwardt2005shortest}
\bibfield{author}{\bibinfo{person}{Karsten~M Borgwardt} {and}
  \bibinfo{person}{Hans-Peter Kriegel}.} \bibinfo{year}{2005}\natexlab{}.
\newblock \showarticletitle{Shortest-path kernels on graphs}. In
  \bibinfo{booktitle}{\emph{ICDM}}. IEEE.
\newblock


\bibitem[\protect\citeauthoryear{Cai, Sheth, Mackey, and Fusi}{Cai
  et~al\mbox{.}}{2020}]%
        {cai2020weighted}
\bibfield{author}{\bibinfo{person}{Diana Cai}, \bibinfo{person}{Rishit Sheth},
  \bibinfo{person}{Lester Mackey}, {and} \bibinfo{person}{Nicolo Fusi}.}
  \bibinfo{year}{2020}\natexlab{}.
\newblock \bibinfo{title}{Weighted Meta-Learning}.
\newblock
\newblock
\showeprint[arxiv]{stat.ML/2003.09465}


\bibitem[\protect\citeauthoryear{Chauhan, Nathani, and Kaul}{Chauhan
  et~al\mbox{.}}{2020}]%
        {Chauhan2020FEW-SHOT}
\bibfield{author}{\bibinfo{person}{Jatin Chauhan}, \bibinfo{person}{Deepak
  Nathani}, {and} \bibinfo{person}{Manohar Kaul}.}
  \bibinfo{year}{2020}\natexlab{}.
\newblock \showarticletitle{FEW-SHOT LEARNING ON GRAPHS VIA SUPER-CLASSES BASED
  ON GRAPH SPECTRAL MEASURES}. In \bibinfo{booktitle}{\emph{ICLR}}.
\newblock


\bibitem[\protect\citeauthoryear{{Dahl}, {Sainath}, and {Hinton}}{{Dahl}
  et~al\mbox{.}}{2013}]%
        {ReLU}
\bibfield{author}{\bibinfo{person}{G.~E. {Dahl}}, \bibinfo{person}{T.~N.
  {Sainath}}, {and} \bibinfo{person}{G.~E. {Hinton}}.}
  \bibinfo{year}{2013}\natexlab{}.
\newblock \showarticletitle{Improving deep neural networks for LVCSR using
  rectified linear units and dropout}. In \bibinfo{booktitle}{\emph{2013 IEEE
  International Conference on Acoustics, Speech and Signal Processing}}.
  \bibinfo{pages}{8609--8613}.
\newblock


\bibitem[\protect\citeauthoryear{Diehl}{Diehl}{2019}]%
        {diehl2019edge}
\bibfield{author}{\bibinfo{person}{Frederik Diehl}.}
  \bibinfo{year}{2019}\natexlab{}.
\newblock \showarticletitle{Edge Contraction Pooling for Graph Neural
  Networks}.
\newblock \bibinfo{journal}{\emph{arXiv preprint arXiv:1905.10990}}
  (\bibinfo{year}{2019}).
\newblock


\bibitem[\protect\citeauthoryear{Du, Wang, Yang, Zhou, and Tang}{Du
  et~al\mbox{.}}{2019}]%
        {recommendation}
\bibfield{author}{\bibinfo{person}{Zhengxiao Du}, \bibinfo{person}{Xiaowei
  Wang}, \bibinfo{person}{Hongxia Yang}, \bibinfo{person}{Jingren Zhou}, {and}
  \bibinfo{person}{Jie Tang}.} \bibinfo{year}{2019}\natexlab{}.
\newblock \showarticletitle{Sequential Scenario-Specific Meta Learner for
  Online Recommendation}. In \bibinfo{booktitle}{\emph{KDD}}.
  \bibinfo{pages}{2895–2904}.
\newblock


\bibitem[\protect\citeauthoryear{Errica, Podda, Bacciu, and Micheli}{Errica
  et~al\mbox{.}}{2020}]%
        {Errica2020A}
\bibfield{author}{\bibinfo{person}{Federico Errica}, \bibinfo{person}{Marco
  Podda}, \bibinfo{person}{Davide Bacciu}, {and} \bibinfo{person}{Alessio
  Micheli}.} \bibinfo{year}{2020}\natexlab{}.
\newblock \showarticletitle{A Fair Comparison of Graph Neural Networks for
  Graph Classification}. In \bibinfo{booktitle}{\emph{ICLR}}.
\newblock


\bibitem[\protect\citeauthoryear{Finn, Abbeel, and Levine}{Finn
  et~al\mbox{.}}{2017}]%
        {FinnAL17}
\bibfield{author}{\bibinfo{person}{Chelsea Finn}, \bibinfo{person}{Pieter
  Abbeel}, {and} \bibinfo{person}{Sergey Levine}.}
  \bibinfo{year}{2017}\natexlab{}.
\newblock \showarticletitle{Model-Agnostic Meta-Learning for Fast Adaptation of
  Deep Networks}. In \bibinfo{booktitle}{\emph{ICML}}.
\newblock


\bibitem[\protect\citeauthoryear{Gao and Ji}{Gao and Ji}{2019}]%
        {unet}
\bibfield{author}{\bibinfo{person}{Hongyang Gao} {and}
  \bibinfo{person}{Shuiwang Ji}.} \bibinfo{year}{2019}\natexlab{}.
\newblock \showarticletitle{Graph U-Nets}. In \bibinfo{booktitle}{\emph{ICML}}
  \emph{(\bibinfo{series}{PMLR})}, \bibfield{editor}{\bibinfo{person}{Kamalika
  Chaudhuri} {and} \bibinfo{person}{Ruslan Salakhutdinov}} (Eds.),
  Vol.~\bibinfo{volume}{97}. \bibinfo{address}{Long Beach, California, USA},
  \bibinfo{pages}{2083--2092}.
\newblock


\bibitem[\protect\citeauthoryear{Garg, Jegelka, and Jaakkola}{Garg
  et~al\mbox{.}}{2020}]%
        {garg2020generalization}
\bibfield{author}{\bibinfo{person}{Vikas~K. Garg}, \bibinfo{person}{Stefanie
  Jegelka}, {and} \bibinfo{person}{Tommi Jaakkola}.}
  \bibinfo{year}{2020}\natexlab{}.
\newblock \bibinfo{title}{Generalization and Representational Limits of Graph
  Neural Networks}.
\newblock
\newblock
\showeprint[arxiv]{cs.LG/2002.06157}


\bibitem[\protect\citeauthoryear{Gidaris and Komodakis}{Gidaris and
  Komodakis}{2019}]%
        {Gidaris_2019_CVPR}
\bibfield{author}{\bibinfo{person}{Spyros Gidaris} {and} \bibinfo{person}{Nikos
  Komodakis}.} \bibinfo{year}{2019}\natexlab{}.
\newblock \showarticletitle{Generating Classification Weights With GNN
  Denoising Autoencoders for Few-Shot Learning}. In
  \bibinfo{booktitle}{\emph{CVPR}}.
\newblock


\bibitem[\protect\citeauthoryear{Hamilton, Ying, and Leskovec}{Hamilton
  et~al\mbox{.}}{2017}]%
        {SAGE}
\bibfield{author}{\bibinfo{person}{Will Hamilton}, \bibinfo{person}{Zhitao
  Ying}, {and} \bibinfo{person}{Jure Leskovec}.}
  \bibinfo{year}{2017}\natexlab{}.
\newblock \showarticletitle{Inductive Representation Learning on Large Graphs}.
\newblock In \bibinfo{booktitle}{\emph{NeurlPS}}.
\newblock


\bibitem[\protect\citeauthoryear{Hou, Zhang, Cheng, Ma, Ma, Chen, and Yang}{Hou
  et~al\mbox{.}}{2020}]%
        {Hou2020Measuring}
\bibfield{author}{\bibinfo{person}{Yifan Hou}, \bibinfo{person}{Jian Zhang},
  \bibinfo{person}{James Cheng}, \bibinfo{person}{Kaili Ma},
  \bibinfo{person}{Richard T.~B. Ma}, \bibinfo{person}{Hongzhi Chen}, {and}
  \bibinfo{person}{Ming-Chang Yang}.} \bibinfo{year}{2020}\natexlab{}.
\newblock \showarticletitle{Measuring and Improving the Use of Graph
  Information in Graph Neural Networks}. In \bibinfo{booktitle}{\emph{ICLR}}.
\newblock


\bibitem[\protect\citeauthoryear{Hu*, Liu*, Gomes, Zitnik, Liang, Pande, and
  Leskovec}{Hu* et~al\mbox{.}}{2020}]%
        {Hu*2020Strategies}
\bibfield{author}{\bibinfo{person}{Weihua Hu*}, \bibinfo{person}{Bowen Liu*},
  \bibinfo{person}{Joseph Gomes}, \bibinfo{person}{Marinka Zitnik},
  \bibinfo{person}{Percy Liang}, \bibinfo{person}{Vijay Pande}, {and}
  \bibinfo{person}{Jure Leskovec}.} \bibinfo{year}{2020}\natexlab{}.
\newblock \showarticletitle{Strategies for Pre-training Graph Neural Networks}.
  In \bibinfo{booktitle}{\emph{ICLR}}.
\newblock


\bibitem[\protect\citeauthoryear{J.}{J.}{1992}]%
        {rl}
\bibfield{author}{\bibinfo{person}{Williams~Ronald J.}}
  \bibinfo{year}{1992}\natexlab{}.
\newblock \showarticletitle{Simple statistical gradient-following algorithms
  for connectionist reinforcement learning}.
\newblock \bibinfo{journal}{\emph{Machine Learning}}  \bibinfo{volume}{8}
  (\bibinfo{year}{1992}), \bibinfo{pages}{229--256}.
\newblock


\bibitem[\protect\citeauthoryear{Kim, Kim, Kim, and Yoo}{Kim
  et~al\mbox{.}}{2019}]%
        {Kim_2019_CVPR}
\bibfield{author}{\bibinfo{person}{Jongmin Kim}, \bibinfo{person}{Taesup Kim},
  \bibinfo{person}{Sungwoong Kim}, {and} \bibinfo{person}{Chang~D. Yoo}.}
  \bibinfo{year}{2019}\natexlab{}.
\newblock \showarticletitle{Edge-Labeling Graph Neural Network for Few-Shot
  Learning}. In \bibinfo{booktitle}{\emph{CVPR}}.
\newblock


\bibitem[\protect\citeauthoryear{Kipf and Welling}{Kipf and Welling}{2017}]%
        {kipf2016semi}
\bibfield{author}{\bibinfo{person}{Thomas~N Kipf} {and} \bibinfo{person}{Max
  Welling}.} \bibinfo{year}{2017}\natexlab{}.
\newblock \showarticletitle{Semi-supervised classification with graph
  convolutional networks}.
\newblock \bibinfo{journal}{\emph{ICLR}} (\bibinfo{year}{2017}).
\newblock


\bibitem[\protect\citeauthoryear{Knyazev, Taylor, and Amer}{Knyazev
  et~al\mbox{.}}{2019}]%
        {NIPS2019_8673}
\bibfield{author}{\bibinfo{person}{Boris Knyazev}, \bibinfo{person}{Graham~W
  Taylor}, {and} \bibinfo{person}{Mohamed Amer}.}
  \bibinfo{year}{2019}\natexlab{}.
\newblock \showarticletitle{Understanding Attention and Generalization in Graph
  Neural Networks}.
\newblock In \bibinfo{booktitle}{\emph{NeurlPS}}.
\newblock


\bibitem[\protect\citeauthoryear{Lee, Im, Jang, Cho, and Chung}{Lee
  et~al\mbox{.}}{2019a}]%
        {MeLU}
\bibfield{author}{\bibinfo{person}{Hoyeop Lee}, \bibinfo{person}{Jinbae Im},
  \bibinfo{person}{Seongwon Jang}, \bibinfo{person}{Hyunsouk Cho}, {and}
  \bibinfo{person}{Sehee Chung}.} \bibinfo{year}{2019}\natexlab{a}.
\newblock \showarticletitle{MeLU: Meta-Learned User Preference Estimator for
  Cold-Start Recommendation}. In \bibinfo{booktitle}{\emph{KDD}} (Anchorage,
  AK, USA) \emph{(\bibinfo{series}{KDD ’19})}.
  \bibinfo{publisher}{Association for Computing Machinery},
  \bibinfo{address}{New York, NY, USA}, \bibinfo{pages}{1073–1082}.
\newblock
\showISBNx{9781450362016}
\urldef\tempurl%
\url{https://doi.org/10.1145/3292500.3330859}
\showDOI{\tempurl}


\bibitem[\protect\citeauthoryear{Lee, Kim, Lee, and Yoon}{Lee
  et~al\mbox{.}}{2017}]%
        {LeeTransfer}
\bibfield{author}{\bibinfo{person}{Jaekoo Lee}, \bibinfo{person}{Hyunjae Kim},
  \bibinfo{person}{Jongsun Lee}, {and} \bibinfo{person}{Sungroh Yoon}.}
  \bibinfo{year}{2017}\natexlab{}.
\newblock \showarticletitle{Transfer Learning for Deep Learning on
  Graph-Structured Data}. In \bibinfo{booktitle}{\emph{AAAI}}.
\newblock


\bibitem[\protect\citeauthoryear{Lee, Lee, and Kang}{Lee
  et~al\mbox{.}}{2019b}]%
        {lee2019self}
\bibfield{author}{\bibinfo{person}{Junhyun Lee}, \bibinfo{person}{Inyeop Lee},
  {and} \bibinfo{person}{Jaewoo Kang}.} \bibinfo{year}{2019}\natexlab{b}.
\newblock \showarticletitle{Self-Attention Graph Pooling}. In
  \bibinfo{booktitle}{\emph{ICML}}. \bibinfo{pages}{3734--3743}.
\newblock


\bibitem[\protect\citeauthoryear{Li, Zhou, Chen, and Li}{Li
  et~al\mbox{.}}{2018}]%
        {LiICML2018}
\bibfield{author}{\bibinfo{person}{Zhenguo Li}, \bibinfo{person}{Fengwei Zhou},
  \bibinfo{person}{Fei Chen}, {and} \bibinfo{person}{Hang Li}.}
  \bibinfo{year}{2018}\natexlab{}.
\newblock \showarticletitle{Meta-SGD: Learning to Learn Quickly for Few Shot
  Learning}. In \bibinfo{booktitle}{\emph{ICML}}.
\newblock


\bibitem[\protect\citeauthoryear{Liu, Zhou, Long, Jiang, Yao, and Zhang}{Liu
  et~al\mbox{.}}{2019}]%
        {Liu2019PrototypePN}
\bibfield{author}{\bibinfo{person}{Lu Liu}, \bibinfo{person}{Tianyi Zhou},
  \bibinfo{person}{Guodong Long}, \bibinfo{person}{Jing Jiang},
  \bibinfo{person}{Lina Yao}, {and} \bibinfo{person}{Chengqi Zhang}.}
  \bibinfo{year}{2019}\natexlab{}.
\newblock \showarticletitle{Prototype Propagation Networks (PPN) for
  Weakly-supervised Few-shot Learning on Category Graph}. In
  \bibinfo{booktitle}{\emph{IJCAI}}.
\newblock


\bibitem[\protect\citeauthoryear{Liu, Zhou, Long, Jiang, and Zhang}{Liu
  et~al\mbox{.}}{2019}]%
        {liu2019GPN}
\bibfield{author}{\bibinfo{person}{Lu Liu}, \bibinfo{person}{Tianyi Zhou},
  \bibinfo{person}{Guodong Long}, \bibinfo{person}{Jing Jiang}, {and}
  \bibinfo{person}{Chengqi Zhang}.} \bibinfo{year}{2019}\natexlab{}.
\newblock \showarticletitle{Learning to Propagate for Graph Meta-Learning}. In
  \bibinfo{booktitle}{\emph{NeurIPS}}.
\newblock


\bibitem[\protect\citeauthoryear{LIU, Zhou, Long, Jiang, and Zhang}{LIU
  et~al\mbox{.}}{2019}]%
        {NIPS2019_8389}
\bibfield{author}{\bibinfo{person}{LU LIU}, \bibinfo{person}{Tianyi Zhou},
  \bibinfo{person}{Guodong Long}, \bibinfo{person}{Jing Jiang}, {and}
  \bibinfo{person}{Chengqi Zhang}.} \bibinfo{year}{2019}\natexlab{}.
\newblock \showarticletitle{Learning to Propagate for Graph Meta-Learning}.
\newblock In \bibinfo{booktitle}{\emph{NeurIPS}}. \bibinfo{publisher}{Curran
  Associates, Inc.}, \bibinfo{pages}{1039--1050}.
\newblock


\bibitem[\protect\citeauthoryear{Müller}{Müller}{1997}]%
        {IPM}
\bibfield{author}{\bibinfo{person}{Alfred Müller}.}
  \bibinfo{year}{1997}\natexlab{}.
\newblock \showarticletitle{Integral Probability Metrics and Their Generating
  Classes of Functions}.
\newblock \bibinfo{journal}{\emph{Advances in Applied Probability}}
  \bibinfo{volume}{29}, \bibinfo{number}{2} (\bibinfo{year}{1997}),
  \bibinfo{pages}{429--443}.
\newblock
\showISSN{00018678}


\bibitem[\protect\citeauthoryear{Nichol, Achiam, and Schulman}{Nichol
  et~al\mbox{.}}{2018}]%
        {nichol2018firstorder}
\bibfield{author}{\bibinfo{person}{Alex Nichol}, \bibinfo{person}{Joshua
  Achiam}, {and} \bibinfo{person}{John Schulman}.}
  \bibinfo{year}{2018}\natexlab{}.
\newblock \bibinfo{title}{On First-Order Meta-Learning Algorithms}.
\newblock
\newblock
\showeprint[arxiv]{cs.LG/1803.02999}


\bibitem[\protect\citeauthoryear{Pan, Li, Ao, Tang, and He}{Pan
  et~al\mbox{.}}{2019}]%
        {adCold-Start}
\bibfield{author}{\bibinfo{person}{Feiyang Pan}, \bibinfo{person}{Shuokai Li},
  \bibinfo{person}{Xiang Ao}, \bibinfo{person}{Pingzhong Tang}, {and}
  \bibinfo{person}{Qing He}.} \bibinfo{year}{2019}\natexlab{}.
\newblock \showarticletitle{Warm Up Cold-Start Advertisements: Improving CTR
  Predictions via Learning to Learn ID Embeddings}. In
  \bibinfo{booktitle}{\emph{SIGIR}} (Paris, France)
  \emph{(\bibinfo{series}{SIGIR’19})}. \bibinfo{publisher}{Association for
  Computing Machinery}, \bibinfo{address}{New York, NY, USA},
  \bibinfo{pages}{695–704}.
\newblock
\showISBNx{9781450361729}
\urldef\tempurl%
\url{https://doi.org/10.1145/3331184.3331268}
\showDOI{\tempurl}


\bibitem[\protect\citeauthoryear{Raghu, Raghu, Bengio, and Vinyals}{Raghu
  et~al\mbox{.}}{2020}]%
        {Raghu2020Rapid}
\bibfield{author}{\bibinfo{person}{Aniruddh Raghu}, \bibinfo{person}{Maithra
  Raghu}, \bibinfo{person}{Samy Bengio}, {and} \bibinfo{person}{Oriol
  Vinyals}.} \bibinfo{year}{2020}\natexlab{}.
\newblock \showarticletitle{Rapid Learning or Feature Reuse? Towards
  Understanding the Effectiveness of MAML}. In
  \bibinfo{booktitle}{\emph{ICLR}}.
\newblock


\bibitem[\protect\citeauthoryear{Ravi and Larochelle}{Ravi and
  Larochelle}{2017}]%
        {RaviICLR2017}
\bibfield{author}{\bibinfo{person}{Sachin Ravi} {and} \bibinfo{person}{Hugo
  Larochelle}.} \bibinfo{year}{2017}\natexlab{}.
\newblock \showarticletitle{Optimization as a Model for Few-Shot Learning}. In
  \bibinfo{booktitle}{\emph{ICLR}}.
\newblock


\bibitem[\protect\citeauthoryear{Riesen and Bunke}{Riesen and Bunke}{2008}]%
        {COIL}
\bibfield{author}{\bibinfo{person}{Kaspar Riesen} {and} \bibinfo{person}{Horst
  Bunke}.} \bibinfo{year}{2008}\natexlab{}.
\newblock \showarticletitle{"IAM Graph Database Repository for Graph Based
  Pattern Recognition and Machine Learning"}. In
  \bibinfo{booktitle}{\emph{Structural, Syntactic, and Statistical Pattern
  Recognition}}, \bibfield{editor}{\bibinfo{person}{Niels da~Vitoria~Lobo},
  \bibinfo{person}{Takis Kasparis}, \bibinfo{person}{Fabio Roli},
  \bibinfo{person}{James~T. Kwok}, \bibinfo{person}{Michael Georgiopoulos},
  \bibinfo{person}{Georgios~C. Anagnostopoulos}, {and} \bibinfo{person}{Marco
  Loog}} (Eds.).
\newblock


\bibitem[\protect\citeauthoryear{Rusu, Rao, Sygnowski, Vinyals, Pascanu,
  Osindero, and Hadsell}{Rusu et~al\mbox{.}}{2019}]%
        {RusuICLR2019}
\bibfield{author}{\bibinfo{person}{Andrei~A. Rusu}, \bibinfo{person}{Dushyant
  Rao}, \bibinfo{person}{Jakub Sygnowski}, \bibinfo{person}{Oriol Vinyals},
  \bibinfo{person}{Razvan Pascanu}, \bibinfo{person}{Simon Osindero}, {and}
  \bibinfo{person}{Raia Hadsell}.} \bibinfo{year}{2019}\natexlab{}.
\newblock \showarticletitle{Meta-Learning with Latent Embedding Optimization}.
  In \bibinfo{booktitle}{\emph{ICLR}}.
\newblock


\bibitem[\protect\citeauthoryear{Santoro, Bartunov, Botvinick, Wierstra, and
  Lillicrap}{Santoro et~al\mbox{.}}{2016}]%
        {SantoroBBWL16}
\bibfield{author}{\bibinfo{person}{Adam Santoro}, \bibinfo{person}{Sergey
  Bartunov}, \bibinfo{person}{Matthew Botvinick}, \bibinfo{person}{Daan
  Wierstra}, {and} \bibinfo{person}{Timothy~P. Lillicrap}.}
  \bibinfo{year}{2016}\natexlab{}.
\newblock \showarticletitle{Meta-Learning with Memory-Augmented Neural
  Networks}. In \bibinfo{booktitle}{\emph{ICML}}.
\newblock


\bibitem[\protect\citeauthoryear{Satorras and Estrach}{Satorras and
  Estrach}{2018}]%
        {SatorrasICLR2018graph}
\bibfield{author}{\bibinfo{person}{Victor~Garcia Satorras} {and}
  \bibinfo{person}{Joan~Bruna Estrach}.} \bibinfo{year}{2018}\natexlab{}.
\newblock \showarticletitle{Few-Shot Learning with Graph Neural Networks}. In
  \bibinfo{booktitle}{\emph{ICLR}}.
\newblock


\bibitem[\protect\citeauthoryear{Shervashidze, Schweitzer, Leeuwen, Mehlhorn,
  and Borgwardt}{Shervashidze et~al\mbox{.}}{2011}]%
        {shervashidze2011weisfeiler}
\bibfield{author}{\bibinfo{person}{Nino Shervashidze}, \bibinfo{person}{Pascal
  Schweitzer}, \bibinfo{person}{Erik Jan~van Leeuwen}, \bibinfo{person}{Kurt
  Mehlhorn}, {and} \bibinfo{person}{Karsten~M Borgwardt}.}
  \bibinfo{year}{2011}\natexlab{}.
\newblock \showarticletitle{Weisfeiler-lehman graph kernels}.
\newblock \bibinfo{journal}{\emph{JMLR}} \bibinfo{volume}{12},
  \bibinfo{number}{Sep} (\bibinfo{year}{2011}), \bibinfo{pages}{2539--2561}.
\newblock


\bibitem[\protect\citeauthoryear{Shervashidze, Vishwanathan, Petri, Mehlhorn,
  and Borgwardt}{Shervashidze et~al\mbox{.}}{2009}]%
        {shervashidze2009efficient}
\bibfield{author}{\bibinfo{person}{Nino Shervashidze}, \bibinfo{person}{SVN
  Vishwanathan}, \bibinfo{person}{Tobias Petri}, \bibinfo{person}{Kurt
  Mehlhorn}, {and} \bibinfo{person}{Karsten Borgwardt}.}
  \bibinfo{year}{2009}\natexlab{}.
\newblock \showarticletitle{Efficient graphlet kernels for large graph
  comparison}. In \bibinfo{booktitle}{\emph{AISTATS}}.
  \bibinfo{pages}{488--495}.
\newblock


\bibitem[\protect\citeauthoryear{Snell, Swersky, and Zemel}{Snell
  et~al\mbox{.}}{2017}]%
        {SnellSZ17}
\bibfield{author}{\bibinfo{person}{Jake Snell}, \bibinfo{person}{Kevin
  Swersky}, {and} \bibinfo{person}{Richard~S. Zemel}.}
  \bibinfo{year}{2017}\natexlab{}.
\newblock \showarticletitle{Prototypical Networks for Few-shot Learning}. In
  \bibinfo{booktitle}{\emph{NIPS}}.
\newblock


\bibitem[\protect\citeauthoryear{{Sriperumbudur}, {Fukumizu}, {Gretton},
  {Schölkopf}, and {Lanckriet}}{{Sriperumbudur} et~al\mbox{.}}{2010}]%
        {NPE}
\bibfield{author}{\bibinfo{person}{B.~K. {Sriperumbudur}}, \bibinfo{person}{K.
  {Fukumizu}}, \bibinfo{person}{A. {Gretton}}, \bibinfo{person}{B.
  {Schölkopf}}, {and} \bibinfo{person}{G.~R.~G. {Lanckriet}}.}
  \bibinfo{year}{2010}\natexlab{}.
\newblock \showarticletitle{Non-parametric estimation of integral probability
  metrics}. In \bibinfo{booktitle}{\emph{2010 IEEE International Symposium on
  Information Theory}}. \bibinfo{pages}{1428--1432}.
\newblock


\bibitem[\protect\citeauthoryear{Vartak, Thiagarajan, Miranda, Bratman, and
  Larochelle}{Vartak et~al\mbox{.}}{2017}]%
        {NIPS2017_7266}
\bibfield{author}{\bibinfo{person}{Manasi Vartak}, \bibinfo{person}{Arvind
  Thiagarajan}, \bibinfo{person}{Conrado Miranda}, \bibinfo{person}{Jeshua
  Bratman}, {and} \bibinfo{person}{Hugo Larochelle}.}
  \bibinfo{year}{2017}\natexlab{}.
\newblock \showarticletitle{A Meta-Learning Perspective on Cold-Start
  Recommendations for Items}.
\newblock In \bibinfo{booktitle}{\emph{NeurIPS}}. \bibinfo{publisher}{Curran
  Associates, Inc.}, \bibinfo{pages}{6904--6914}.
\newblock


\bibitem[\protect\citeauthoryear{Veličković, Cucurull, Casanova, Romero,
  Liò, and Bengio}{Veličković et~al\mbox{.}}{2018}]%
        {GAT}
\bibfield{author}{\bibinfo{person}{Petar Veličković},
  \bibinfo{person}{Guillem Cucurull}, \bibinfo{person}{Arantxa Casanova},
  \bibinfo{person}{Adriana Romero}, \bibinfo{person}{Pietro Liò}, {and}
  \bibinfo{person}{Yoshua Bengio}.} \bibinfo{year}{2018}\natexlab{}.
\newblock \showarticletitle{Graph Attention Networks}. In
  \bibinfo{booktitle}{\emph{ICLR}}.
\newblock


\bibitem[\protect\citeauthoryear{Veličković, Fedus, Hamilton, Liò, Bengio,
  and Hjelm}{Veličković et~al\mbox{.}}{2019}]%
        {2018deep}
\bibfield{author}{\bibinfo{person}{Petar Veličković},
  \bibinfo{person}{William Fedus}, \bibinfo{person}{William~L. Hamilton},
  \bibinfo{person}{Pietro Liò}, \bibinfo{person}{Yoshua Bengio}, {and}
  \bibinfo{person}{R~Devon Hjelm}.} \bibinfo{year}{2019}\natexlab{}.
\newblock \showarticletitle{Deep Graph Infomax}. In
  \bibinfo{booktitle}{\emph{ICLR}}.
\newblock


\bibitem[\protect\citeauthoryear{Vinyals, Blundell, Lillicrap, Kavukcuoglu, and
  Wierstra}{Vinyals et~al\mbox{.}}{2016}]%
        {VinyalsBLKW16}
\bibfield{author}{\bibinfo{person}{Oriol Vinyals}, \bibinfo{person}{Charles
  Blundell}, \bibinfo{person}{Tim Lillicrap}, \bibinfo{person}{Koray
  Kavukcuoglu}, {and} \bibinfo{person}{Daan Wierstra}.}
  \bibinfo{year}{2016}\natexlab{}.
\newblock \showarticletitle{Matching Networks for One Shot Learning}. In
  \bibinfo{booktitle}{\emph{NeurIPS}}.
\newblock


\bibitem[\protect\citeauthoryear{Xu, Hu, Leskovec, and Jegelka}{Xu
  et~al\mbox{.}}{2019}]%
        {xu2018how}
\bibfield{author}{\bibinfo{person}{Keyulu Xu}, \bibinfo{person}{Weihua Hu},
  \bibinfo{person}{Jure Leskovec}, {and} \bibinfo{person}{Stefanie Jegelka}.}
  \bibinfo{year}{2019}\natexlab{}.
\newblock \showarticletitle{How Powerful are Graph Neural Networks?}. In
  \bibinfo{booktitle}{\emph{ICLR}}.
\newblock


\bibitem[\protect\citeauthoryear{Yang, Li, Zhang, Zhou, Zhou, and Liu}{Yang
  et~al\mbox{.}}{2020}]%
        {yang2020dpgn}
\bibfield{author}{\bibinfo{person}{Ling Yang}, \bibinfo{person}{Liangliang Li},
  \bibinfo{person}{Zilun Zhang}, \bibinfo{person}{Xinyu Zhou},
  \bibinfo{person}{Erjin Zhou}, {and} \bibinfo{person}{Yu Liu}.}
  \bibinfo{year}{2020}\natexlab{}.
\newblock \bibinfo{title}{DPGN: Distribution Propagation Graph Network for
  Few-shot Learning}.
\newblock
\newblock
\showeprint[arxiv]{cs.CV/2003.14247}


\bibitem[\protect\citeauthoryear{Yao, Wu, Tao, Li, Ding, Li, and Li}{Yao
  et~al\mbox{.}}{2020a}]%
        {Yao2020Automated}
\bibfield{author}{\bibinfo{person}{Huaxiu Yao}, \bibinfo{person}{Xian Wu},
  \bibinfo{person}{Zhiqiang Tao}, \bibinfo{person}{Yaliang Li},
  \bibinfo{person}{Bolin Ding}, \bibinfo{person}{Ruirui Li}, {and}
  \bibinfo{person}{Zhenhui Li}.} \bibinfo{year}{2020}\natexlab{a}.
\newblock \showarticletitle{Automated Relational Meta-learning}. In
  \bibinfo{booktitle}{\emph{ICLR}}.
\newblock


\bibitem[\protect\citeauthoryear{Yao, Zhang, Wei, Jiang, Wang, Huang, Chawla,
  and Li}{Yao et~al\mbox{.}}{2020b}]%
        {yao2019graph}
\bibfield{author}{\bibinfo{person}{Huaxiu Yao}, \bibinfo{person}{Chuxu Zhang},
  \bibinfo{person}{Ying Wei}, \bibinfo{person}{Meng Jiang},
  \bibinfo{person}{Suhang Wang}, \bibinfo{person}{Junzhou Huang},
  \bibinfo{person}{Nitesh~V. Chawla}, {and} \bibinfo{person}{Zhenhui Li}.}
  \bibinfo{year}{2020}\natexlab{b}.
\newblock \showarticletitle{Graph Few-shot Learning via Knowledge Transfer}. In
  \bibinfo{booktitle}{\emph{AAAI}}.
\newblock


\bibitem[\protect\citeauthoryear{Yin, Tucker, Zhou, Levine, and Finn}{Yin
  et~al\mbox{.}}{2020}]%
        {Yin2020Meta-Learning}
\bibfield{author}{\bibinfo{person}{Mingzhang Yin}, \bibinfo{person}{George
  Tucker}, \bibinfo{person}{Mingyuan Zhou}, \bibinfo{person}{Sergey Levine},
  {and} \bibinfo{person}{Chelsea Finn}.} \bibinfo{year}{2020}\natexlab{}.
\newblock \showarticletitle{Meta-Learning without Memorization}. In
  \bibinfo{booktitle}{\emph{ICLR}}.
\newblock


\bibitem[\protect\citeauthoryear{Zhang, Zhang, and Ye}{Zhang
  et~al\mbox{.}}{2012}]%
        {10.5555/2999325.2999503}
\bibfield{author}{\bibinfo{person}{Chao Zhang}, \bibinfo{person}{Lei Zhang},
  {and} \bibinfo{person}{Jieping Ye}.} \bibinfo{year}{2012}\natexlab{}.
\newblock \showarticletitle{Generalization Bounds for Domain Adaptation}. In
  \bibinfo{booktitle}{\emph{NeurIPS}} (Lake Tahoe, Nevada)
  \emph{(\bibinfo{series}{NIPS’12})}. \bibinfo{publisher}{Curran Associates
  Inc.}, \bibinfo{address}{Red Hook, NY, USA}, \bibinfo{pages}{3320–3328}.
\newblock


\bibitem[\protect\citeauthoryear{Zhang, Liu, Zhou, Xu, and He}{Zhang
  et~al\mbox{.}}{2018b}]%
        {zhang2018on}
\bibfield{author}{\bibinfo{person}{Pengchuan Zhang}, \bibinfo{person}{Qiang
  Liu}, \bibinfo{person}{Dengyong Zhou}, \bibinfo{person}{Tao Xu}, {and}
  \bibinfo{person}{Xiaodong He}.} \bibinfo{year}{2018}\natexlab{b}.
\newblock \showarticletitle{On the Discrimination-Generalization Trade off in
  {GAN}s}. In \bibinfo{booktitle}{\emph{International Conference on Learning
  Representations}}.
\newblock


\bibitem[\protect\citeauthoryear{ZHANG, Che, Ghahramani, Bengio, and
  Song}{ZHANG et~al\mbox{.}}{2018}]%
        {NIPS2018_7504}
\bibfield{author}{\bibinfo{person}{Ruixiang ZHANG}, \bibinfo{person}{Tong Che},
  \bibinfo{person}{Zoubin Ghahramani}, \bibinfo{person}{Yoshua Bengio}, {and}
  \bibinfo{person}{Yangqiu Song}.} \bibinfo{year}{2018}\natexlab{}.
\newblock \showarticletitle{MetaGAN: An Adversarial Approach to Few-Shot
  Learning}.
\newblock In \bibinfo{booktitle}{\emph{NeurIPS}}. \bibinfo{publisher}{Curran
  Associates, Inc.}, \bibinfo{pages}{2365--2374}.
\newblock


\bibitem[\protect\citeauthoryear{Zhang, Bu, Ester, Zhang, Yao, Yu, and
  Wang}{Zhang et~al\mbox{.}}{2019}]%
        {zhang2019hierarchical}
\bibfield{author}{\bibinfo{person}{Zhen Zhang}, \bibinfo{person}{Jiajun Bu},
  \bibinfo{person}{Martin Ester}, \bibinfo{person}{Jianfeng Zhang},
  \bibinfo{person}{Chengwei Yao}, \bibinfo{person}{Zhi Yu}, {and}
  \bibinfo{person}{Can Wang}.} \bibinfo{year}{2019}\natexlab{}.
\newblock \showarticletitle{Hierarchical Graph Pooling with Structure
  Learning}.
\newblock \bibinfo{journal}{\emph{arXiv preprint arXiv:1911.05954}}
  (\bibinfo{year}{2019}).
\newblock


\bibitem[\protect\citeauthoryear{Zhang, Cui, and Zhu}{Zhang
  et~al\mbox{.}}{2018a}]%
        {zhang2018deep}
\bibfield{author}{\bibinfo{person}{Ziwei Zhang}, \bibinfo{person}{Peng Cui},
  {and} \bibinfo{person}{Wenwu Zhu}.} \bibinfo{year}{2018}\natexlab{a}.
\newblock \showarticletitle{Deep learning on graphs: A survey}.
\newblock \bibinfo{journal}{\emph{arXiv preprint arXiv:1812.04202}}
  (\bibinfo{year}{2018}).
\newblock


\bibitem[\protect\citeauthoryear{Zhou, Cao, Zhang, Trajcevski, Zhong, and
  Geng}{Zhou et~al\mbox{.}}{2019}]%
        {meta_gnn}
\bibfield{author}{\bibinfo{person}{Fan Zhou}, \bibinfo{person}{Chengtai Cao},
  \bibinfo{person}{Kunpeng Zhang}, \bibinfo{person}{Goce Trajcevski},
  \bibinfo{person}{Ting Zhong}, {and} \bibinfo{person}{Ji Geng}.}
  \bibinfo{year}{2019}\natexlab{}.
\newblock \showarticletitle{Meta-GNN: On Few-Shot Node Classification in Graph
  Meta-Learning}. In \bibinfo{booktitle}{\emph{CIKM}} (Beijing, China)
  \emph{(\bibinfo{series}{CIKM ’19})}. \bibinfo{publisher}{Association for
  Computing Machinery}, \bibinfo{address}{New York, NY, USA},
  \bibinfo{pages}{2357–2360}.
\newblock
\showISBNx{9781450369763}


\bibitem[\protect\citeauthoryear{Zhou, Cui, Zhang, Yang, Liu, and Sun}{Zhou
  et~al\mbox{.}}{2018}]%
        {zhou2018graph}
\bibfield{author}{\bibinfo{person}{Jie Zhou}, \bibinfo{person}{Ganqu Cui},
  \bibinfo{person}{Zhengyan Zhang}, \bibinfo{person}{Cheng Yang},
  \bibinfo{person}{Zhiyuan Liu}, {and} \bibinfo{person}{Maosong Sun}.}
  \bibinfo{year}{2018}\natexlab{}.
\newblock \showarticletitle{Graph neural networks: A review of methods and
  applications}.
\newblock \bibinfo{journal}{\emph{arXiv preprint arXiv:1812.08434}}
  (\bibinfo{year}{2018}).
\newblock


\bibitem[\protect\citeauthoryear{Zintgraf, Shiarlis, Kurin, Hofmann, and
  Whiteson}{Zintgraf et~al\mbox{.}}{2019}]%
        {zintgraf2019fast}
\bibfield{author}{\bibinfo{person}{Luisa Zintgraf}, \bibinfo{person}{Kyriacos
  Shiarlis}, \bibinfo{person}{Vitaly Kurin}, \bibinfo{person}{Katja Hofmann},
  {and} \bibinfo{person}{Shimon Whiteson}.} \bibinfo{year}{2019}\natexlab{}.
\newblock \showarticletitle{Fast Context Adaptation via Meta-Learning}. In
  \bibinfo{booktitle}{\emph{ICML}}.
\newblock


\end{thebibliography}
\end{document}